\documentclass{article}

\usepackage{microtype}
\usepackage{graphicx}
\usepackage{subcaption}
\usepackage{booktabs} 
\usepackage{multirow} 

\usepackage{hyperref}

\usepackage[preprint]{icml2026}

\usepackage{amsmath}
\usepackage{amssymb}
\usepackage{mathtools}
\usepackage{amsthm}

\usepackage[capitalize,noabbrev]{cleveref}

\theoremstyle{plain}
\newtheorem{theorem}{Theorem}[section]

\newtheorem{lemma}[theorem]{Lemma}
\newtheorem{corollary}[theorem]{Corollary}
\theoremstyle{definition}
\newtheorem{definition}[theorem]{Definition}
\newtheorem{assumption}[theorem]{Assumption}
\theoremstyle{remark}
\newtheorem{remark}[theorem]{Remark}

\usepackage[textsize=tiny]{todonotes}

\icmltitlerunning{Robust and Efficient Zeroth-Order LLM Fine-Tuning via Adaptive Bayesian Subspace Optimizer}

\begin{document}

\twocolumn[
  \icmltitle{Robust and Efficient Zeroth-Order LLM Fine-Tuning via\\ Adaptive Bayesian Subspace Optimizer}



  \icmlsetsymbol{equal}{*}
  \icmlsetsymbol{corresp}{*}

  \begin{icmlauthorlist}
    \icmlauthor{Jian Feng}{sysu}
    \icmlauthor{Zhihong Huang}{sysu,corresp}
  \end{icmlauthorlist}

  \icmlaffiliation{sysu}{School of Mathematics, Sun Yat-sen University, Guangzhou, China}

  \icmlcorrespondingauthor{Zhihong Huang}{stswzh@mail.sysu.edu.cn}  

  \icmlkeywords{Zeroth-Order Optimization, Large Language Models, Bayesian Inference, Kalman Filtering, Memory-Efficient Fine-Tuning}

  \vskip 0.3in
]



\printAffiliationsAndNotice{}  

\begin{abstract}
Fine-tuning large language models (LLMs) with zeroth-order (ZO) optimization reduces memory by approximating gradients through function evaluations. However, existing methods essentially perform updates in a one-dimensional space, and suffer from collapse or substantial performance degradation under low-precision training. We introduce BSZO, an adaptive  \textbf{B}ayesian \textbf{S}ubspace \textbf{Z}eroth-Order \textbf{O}ptimizer, which applies Kalman filtering to combine finite-difference information across multiple perturbation directions within a subspace. By treating each finite-difference measurement as a noisy observation, BSZO builds a posterior distribution over the subspace-projected gradient and updates it through Bayesian inference, with a residual-based adaptive mechanism to adapt to noise variations. Theoretical analysis shows that BSZO improves the convergence rate by a factor of $k/\gamma$ compared to standard ZO methods. Experiments on RoBERTa, Mistral, and OPT models show that BSZO outperforms the baselines across various tasks, achieving up to 6.67\% absolute average improvement on OPT-13B while remaining robust under fp16/bf16 precision and keeping memory usage close to inference-only baselines (1.00$\times$--1.08$\times$ of MeZO).
\end{abstract}

\section{Introduction}
\label{sec:intro}

Large language models (LLMs) are getting increasingly important in natural language understanding and generation \cite{devlin2019bert, brown2020language, touvron2023llama}. However, adapting these models to downstream tasks through fine-tuning remains challenging due to their large scale. The standard approach, using first-order optimizers like Adam, requires consuming a large amount of GPU memory. For a 13B-parameter model, this translates to over 100GB of GPU memory, roughly 10$\times$ the cost of inference alone \cite{malladi2023fine}. Such requirements put full fine-tuning out of reach for most people, no matter in academia or industry.

Several strategies have been proposed to reduce memory burden. Parameter-efficient fine-tuning (PEFT) methods, including LoRA \cite{hu2022lora} and Adapters \cite{houlsby2019parameter}, freeze the base model and only update a small set of additional parameters. But these methods still rely on backpropagation and may underperform full fine-tuning on difficult tasks. An alternative direction is zeroth-order (ZO) optimization, which estimates gradients using only forward passes. MeZO \cite{malladi2023fine} demonstrated that this approach can match the memory footprint of inference, while achieving reasonable accuracy. The catch? ZO methods converge slowly and require significantly more iterations than their first-order counterparts, due to the high variance inherent in finite-difference gradient estimates.

This raises a question: how can we achieve a better trade-off between convergence speed and memory usage? We observe that the existing ZO methods have three main weaknesses. First, most existing ZO optimizers essentially perform updates along a single random direction within each batch. Even with increased forward passes and perturbation directions, they process each perturbation in isolation, simply averaging or using them independently—throwing away information about how these measurements relate to each other. Second, the noise level in ZO estimates varies significantly during training, yet most methods do not account for this effect. This rigidity leads to poor adaptation: updates may oscillate wildly around local minima, jump out of the basin, and finally cause training collapse. Moreover, reduced-precision training (fp16/bf16) can cause these methods to collapse or suffer substantial performance degradation, as we show in Figure~\ref{fig:motivation} and Table~\ref{tab:llm-results}.

\begin{figure*}[t]
  \centering
  \includegraphics[width=\textwidth]{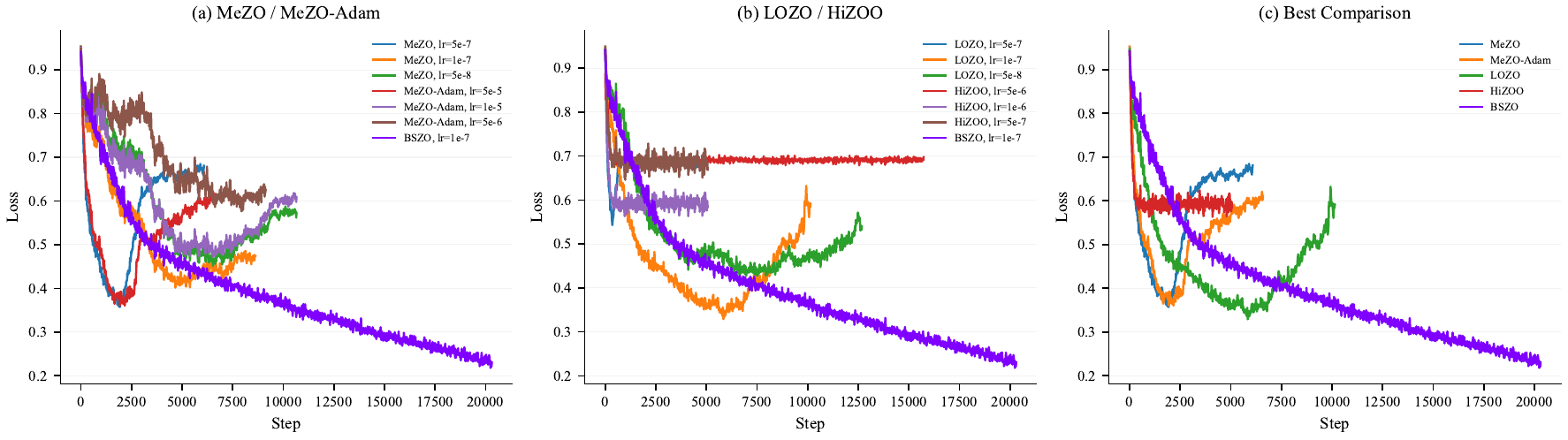}
  \caption{Training loss on SST-2 with OPT-13B under bf16 precision. (a)--(b) Existing ZO methods exhibit erratic loss curves across different learning rates, with some runs failing to converge or even diverging. BSZO achieves smooth and steady convergence. (c) Comparison under each method's best-tuned learning rate.}
  \label{fig:motivation}
\end{figure*}

We propose \textbf{Bayesian Subspace Zeroth-order Optimization (BSZO)} to address these limitations. The main idea is to treat gradient estimation as an inference problem. At each step, we sample $k$ random directions to form a low-dimensional subspace \cite{zhang2025scalable} and model the projected gradient as a latent variable. Instead of treating each finite-difference query as providing an independent estimate, we use Kalman filtering to aggregate observations—essentially asking: given what we have measured so far, what is our best guess of the true gradient? This Bayesian formulation accounts for measurement noise and produces more accurate estimates from the same number of forward passes. We further introduce an adaptive mechanism that tracks prediction residuals and adjusts the noise variance on the fly, allowing the algorithm to respond to changing curvature conditions during training.

Our contributions can be summarized as follows:
\begin{enumerate}
    \item We propose BSZO, a zeroth-order optimizer that uses Bayesian inference to aggregate gradient information across multiple perturbation directions within a subspace. To our knowledge, this is the first application of Bayesian inference and Kalman filtering to ZO optimization for LLMs.

    \item We design a residual-based adaptive scheme that enables BSZO to adjust the parameter update scale adaptively without manual tuning.

    \item We analyze the convergence of BSZO and show that the rate improves by a factor of $k/\gamma$ compared to standard ZO methods.

    \item Experiments on multiple LLMs and benchmarks show that BSZO achieves strong performance across diverse tasks while remaining robust under low-precision training and maintaining memory consumption comparable to MeZO.
\end{enumerate}

\section{Related Work}
\label{sec:related}

\paragraph{Zeroth-Order Optimization for LLMs.}
Classical derivative-free methods achieve strong sample efficiency via surrogate modeling, but their per-iteration cost grows rapidly with dimension, making them impractical at LLM scale \cite{zhang2025scalable}. The SPSA estimator \cite{spall1992multivariate} offers a scalable alternative by approximating gradients through random perturbations. Building on this, MeZO \cite{malladi2023fine} introduced memory-efficient ZO fine-tuning for LLMs, matching inference-time memory by regenerating perturbations from random seeds. Follow-up methods target different bottlenecks: Sparse-MeZO \cite{liu2024sparse} restricts updates to influential parameters, HiZOO \cite{zhao2025hizoo} leverages diagonal Hessian estimates for adaptive preconditioning, LOZO \cite{chen2024enhancing} exploits low-rank gradient structure, and TeZO \cite{sun2025tezo} captures temporal correlations across iterations. Despite these advances, most methods adhere to the ``one batch, one update'' paradigm, overlooking the possibility that multiple function evaluations within a batch could support multiple parameter updates. Moreover, some of these methods incur substantial memory overhead; while still lower than full fine-tuning, this conflicts with the original motivation of ZO optimization---minimizing memory consumption. Since low-precision fine-tuning is essential in memory-constrained scenarios, the robustness of these methods also warrants further evaluation.

\paragraph{Population-Based Gradient Estimation.}
An alternative strategy evaluates multiple perturbations per iteration and aggregates them into a single update. Evolution Strategies \cite{salimans2017evolution} and Augmented Random Search \cite{mania2018simple} popularized this paradigm in reinforcement learning. However, these methods typically require a large number of function evaluations per batch to obtain reliable gradient estimates. Given that each forward pass through an LLM is already computationally expensive, such sample-intensive approaches become impractical for language model fine-tuning. This raises a natural question: how can we extract more information from a limited number of function evaluations? Our work addresses this by treating finite-difference measurements as noisy linear observations of the underlying gradient and applying Bayesian inference to fuse information across directions.

\paragraph{Bayesian Inference for Optimization.}
Bayesian methods provide a principled way to integrate observations with prior knowledge while quantifying uncertainty. Kalman filtering \cite{kalman1960new} is the canonical example: it sequentially updates a Gaussian belief over a hidden state as new measurements arrive. Gaussian processes extend this idea to function-space modeling and underpin Bayesian optimization \cite{shahriari2016taking,williams2006gaussian}. Our work adapts the Kalman perspective to ZO gradient estimation: we model the projected gradient as a hidden state, interpret each perturbation query as a noisy linear measurement, and update a posterior that pools information across all sampled directions within an iteration. Leveraging the flexibility of the Bayesian framework, we further design an adaptive residual mechanism that effectively fuses both historical and current-batch information. This yields improved gradient estimates without additional memory overhead.

\section{Method}
\label{sec:method}

\begin{table*}[t]
  \caption{Test accuracy (\%) on RoBERTa-large (355M). We report the mean$\pm$std over 5 runs. The top two results are highlighted in \textbf{bold}. BSZO-B is the baseline version of BSZO without caching optimization.}
  \label{tab:roberta}
  \begin{center}
    \begin{sc}
      \begin{tabular}{lcccccc}
        \toprule
        Method & SST-2 & RTE & CB & WIC & TREC & Avg \\
        \midrule
        MeZO      & 92.22 ($\pm$0.42) & 66.35 ($\pm$3.06) & \textbf{86.07} ($\pm$5.56) & 55.20 ($\pm$3.73) & \textbf{85.36} ($\pm$2.33) & 77.04 \\
        MeZO-Adam & \textbf{92.34} ($\pm$0.50) & 63.61 ($\pm$\textbf{1.41}) & 81.07 ($\pm$2.71) & 52.85 ($\pm$4.19) & 78.80 ($\pm$5.76) & 73.73 \\
        HiZOO     & 91.44 ($\pm$0.45) & 59.21 ($\pm$2.46) & 76.43 ($\pm$1.96) & 53.60 ($\pm$2.93) & 63.44 ($\pm$2.61) & 68.82 \\
        LOZO      & 91.83 ($\pm$\textbf{0.30}) & 62.60 ($\pm$2.31) & 84.29 ($\pm$3.87) & 54.20 ($\pm$\textbf{1.32}) & 77.76 ($\pm$2.15) & 74.14 \\
        BSZO      & \textbf{92.66} ($\pm$\textbf{0.21}) & \textbf{67.80} ($\pm$\textbf{1.52}) & \textbf{85.71} ($\pm$\textbf{1.79}) & \textbf{56.05} ($\pm$1.47) & 84.16 ($\pm$\textbf{0.54}) & \textbf{77.28} \\
        BSZO-B    & 92.27 ($\pm$0.41) & \textbf{68.38} ($\pm$1.94) & 84.29 ($\pm$\textbf{1.49}) & \textbf{57.21} ($\pm$\textbf{0.98}) & \textbf{84.80} ($\pm$\textbf{1.57}) & \textbf{77.39} \\
        \bottomrule
      \end{tabular}
    \end{sc}
  \end{center}
\end{table*}

\begin{algorithm}[t]
  \caption{Bayesian Subspace Zeroth-Order Optimization (BSZO)}
  \label{alg:bszo}
  \begin{algorithmic}
    \STATE {\bfseries Input:} parameters $\theta$, learning rate $\eta$, perturbation scale $\varepsilon$, subspace dimension $k$, sampling steps $m$, prior variance $\sigma_p^2$, noise variance $\sigma_e^2$, smoothing factor $\alpha$, max step $T$
    \FOR{$t=1$ {\bfseries to} $T$}
    \STATE Sample $k$ random seeds $\{s_i\}_{i=1}^k$
    \STATE Initialize $\mu \leftarrow \mathbf{0}_k$, $\Sigma \leftarrow \sigma_p^2 I_k$, $f_0 \leftarrow \mathcal{L}(\theta)$
    \STATE Initialize cache $Y \leftarrow \{\}$
    \FOR{$\tau=1$ {\bfseries to} $m$}
      \IF{$\tau \leq k$}
        \STATE $d \leftarrow e_\tau$
        \STATE $\theta \leftarrow \theta + \varepsilon \cdot \textsc{Randn}(n, s_\tau)$
        \STATE $y \leftarrow (\mathcal{L}(\theta) - f_0) / \varepsilon$
        \STATE $\theta \leftarrow \theta - \varepsilon \cdot \textsc{Randn}(n, s_\tau)$
        \STATE $Y[\tau] \leftarrow y$ \hfill $\triangleright$ Cache directional derivative
      \ELSE
         \STATE $r \leftarrow (y - d^\top \mu) / \|d\|$, \quad $\sigma_e^2 \leftarrow (1-\alpha)\sigma_e^2 + \alpha r^2$
        \STATE $j \leftarrow \arg\max_i \Sigma_{ii}$ \hfill $\triangleright$ Find max uncertainty axis
        \STATE $d \leftarrow e_j$
        \STATE $y \leftarrow Y[j]$ \hfill $\triangleright$ Reuse cached value
      \ENDIF
      \STATE $K \leftarrow \Sigma d / (d^\top \Sigma d + \sigma_e^2)$
      \STATE $\mu \leftarrow \mu + K (y - d^\top \mu)$, \quad $\Sigma \leftarrow \Sigma - K d^\top \Sigma$
    \ENDFOR
    \FOR{$i=1$ {\bfseries to} $k$}
      \STATE $\theta \leftarrow \theta - \eta \cdot \mu_i \cdot \textsc{Randn}(n, s_i)$
    \ENDFOR
  \ENDFOR
  \STATE {\bfseries return} $\theta$
\STATE {\small $\textsc{Randn}(n,s)$: returns $n$-dim Gaussian vector seeded by $s$}
  \end{algorithmic}
\end{algorithm}

In this section, we present the Bayesian Subspace Zeroth-order Optimization (BSZO) algorithm, which controls the step size of subspace by the Bayesian method.

\subsection{Preliminaries}
\label{sec:preliminaries}
We consider the stochastic optimization problem:
\begin{equation}
  \min_{\theta \in \mathbb{R}^n} \mathcal{L}(\theta) := \mathbb{E}_{\xi \sim \mathcal{D}}[\mathcal{L}(\theta; \xi)],
\end{equation}
where $\theta \in \mathbb{R}^n$ denotes the model parameters, $\mathcal{D}$ is the training dataset, and $\mathcal{L}(\theta; \xi)$ is the loss on a minibatch $\xi$. We denote the optimal value by $\mathcal{L}^* := \min_{\theta} \mathcal{L}(\theta)$.

\begin{assumption}
  \label{ass:l-smooth}
  The function $\mathcal{L}$ is $L$-smooth, i.e., there exists $L>0$ such that for all $\theta, \theta'\in \mathbb{R}^n$, 
  \begin{equation}
    \| \mathcal{L}(\theta) - \mathcal{L}(\theta') \| \leq L \| \theta - \theta' \|.
  \end{equation}
  Equivalently,
  \begin{equation}
    \mathcal{L}(\theta)\le\mathcal{L}(\theta') + \nabla \mathcal{L}(\theta')^\top (\theta - \theta') + \frac{L}{2} \| \theta - \theta' \|^2.
  \end{equation}
\end{assumption}

\begin{assumption}
  \label{ass:bounded-variance}
  The stochastic gradient $\nabla \mathcal{L}(\theta, \xi)$ has bounded variance, i.e., there exists $\sigma_g^2\ge 0$ such that:
  \begin{equation}
    \mathbb{E}_\xi[\|\nabla \mathcal{L}(\theta; \xi) - \nabla \mathcal{L}(\theta)\|^2] \leq \sigma_g^2, \quad \forall \theta \in \mathbb{R}^n
  \end{equation}
\end{assumption}

\begin{definition}
  \label{def:subspace-basis matrix}
  Given a set of $k$ perturbation vectors $\{z_1, z_2, \ldots, z_k\}$, where $z_i\in\mathbb{R}^n$ is from Gaussian distribution $\mathcal{N}(0, I_n)$, define the subspace basis matrix $B = [z_1, z_2, \ldots, z_k] \in \mathbb{R}^{n\times k}$. 
\end{definition}

\begin{definition}
  \label{def:one-sided-diff}
  The one-side difference of $\mathcal{L}$ along the displacement $d\in\mathbb{R}^k$ in subspace $B$ on minibatch $\xi$ is defined as follows:
  \begin{equation}
    \hat{y}(\theta; \xi, d) = \frac{\mathcal{L}(\theta + \varepsilon Bd; \xi) - \mathcal{L}(\theta; \xi)}{\varepsilon},
  \end{equation}
  where $\varepsilon > 0$ is a small constant.
\end{definition}

\subsection{Bayesian Gradient Estimation}
\label{sec:bayesian}
For the $\mathcal{L}(\theta+\varepsilon Bd)$, the subspace gradient can be obtained through the chain rule:
\begin{equation}
  g_d := \nabla_d \mathcal{L}(\theta+\varepsilon Bd\mid d=0) = \varepsilon B^T g,
\end{equation} 
where $g:=\nabla \mathcal{L}$ is the real gradient of $\mathcal{L}$.
In order to keep numerical accuracy controllable, we introduce the concept of normalized subspace gradient as $\tilde{g} := B^\top g = \frac{g_s}{\varepsilon}\in\mathbb{R}^k$.

\begin{lemma}
  \label{lem:expectation-of-onesided-diff}
  For any direction $d\in\mathbb{R}^k$ of subspace $B$, the expectation of one-side difference $\hat{y}(d)$ satisfies:
  \begin{equation}
    \mathbb{E}[\hat{y}(d)] = d^\top B^\top g + O(\varepsilon L)\approx d^\top \tilde{g}.
  \end{equation}
\end{lemma}

Based on Lemma\ref{lem:expectation-of-onesided-diff}, we can model the one-side difference $\hat{y}(d)$ as a linear observation of the normalized subspace gradient $\tilde{g}$ with Gaussian noise:
\begin{equation}
  \hat{y}(d) = d^\top \tilde{g} + \nu, \quad \nu \sim \mathcal{N}(0, \sigma_{e}^2 \|d\|^2),
\end{equation}
where $\nu$ represents comprehensive noise term. The justification of the variance definition is provided in Appendix~\ref{app:noise-variance}. Then, we adopt a Bayesian approach by placing a Gaussian prior on $\tilde{g}$, i.e., $\tilde{g} \sim \mathcal{N}(0, \sigma_p^2 I_k)$ which make the posterior computable in closed-form \cite{kalman1960new}.

\subsection{Posterior Update In Subspace}
According to the standard Bayesian linear regression theory \cite{williams2006gaussian}, after $m$ perturbations and observations $(d^{(1)}, \hat{y}^{(1)}), \ldots, (d^{(m)}, \hat{y}^{(m)})$, the posterior $\tilde{g} | Y \sim \mathcal{N}(\mu^{(m)}, \Sigma^{(m)})$ is also a Gaussian distribution, where
\begin{equation}
\begin{aligned}
  \Sigma^{(m)} &= \left( \sigma_p^{-2} I_k + D^\top R^{-1} D \right)^{-1}, \\
  \mu^{(m)} &= \Sigma^{(m)} D^\top R^{-1} Y.
\end{aligned}
\end{equation}
Here, $D = [d^{(1)}, \ldots, d^{(m)}]^\top \in \mathbb{R}^{m \times k}$ is the design matrix, $Y = [\hat{y}^{(1)}, \ldots, \hat{y}^{(m)}]^\top \in \mathbb{R}^m$ is the observation vector, and $R = \text{diag}(\sigma_e^2 \|d^{(1)}\|^2, \ldots, \sigma_e^2 \|d^{(m)}\|^2)$ is the noise covariance matrix. When $m > k$ or $\Sigma$ is already full-rank, we set the new sampling direction to the principal eigenvector of the covariance matrix, i.e., $d^{(j)} = v_{\max}(\Sigma^{(j-1)})$.\\
After getting the posterior mean $\mu^{(m)}$, we can use it as the final displacement in subspace $B$, which means the parameters updated by:
\begin{equation}
  \Delta\theta = - \eta B \mu^{(k)},
\end{equation}
where $\eta>0$ is learning rate. In this way, we can use the finite k forward passes to update the parameters k times, with $\mu^{(k)}$ controlling the step size in subspace. This means that, for the same batch, the parameters move along a "diagonal" direction rather than a single direction.\\
\begin{corollary}
  \label{cor:reduced-posterior}
  Under \textbf{coordinate-axis sampling}, i.e., $m=k$ and $d^{(i)}=e_i$ (the $i$-th standard basis vector), then the posterior mean and covariance reduce to:
\begin{equation}
\begin{aligned}
  \Sigma^{(k)} &= \gamma I_k, \\
  \mu^{(k)} &= \gamma Y,
\end{aligned}
\end{equation}
\end{corollary}
where $\gamma := \frac{\sigma_p^2}{\sigma_p^2 + \sigma_e^2} \in (0, 1)$ is the shrinkage factor.\\
Corollary\ref{cor:reduced-posterior} simplifies the form of the posterior distribution, thereby making the analysis and update easier. Thus, we adopt coordinate-axis sampling as the default sampling strategy in BSZO (for the first $k$ sampling directions).
\begin{theorem}
  \label{thm:expected-update-direction}
  Let $\Delta\theta = -\eta B \mu^{(k)}$. Under Assumptions \ref{ass:l-smooth} and Assumption\ref{ass:bounded-variance}, we have:
  \begin{equation}
    \mathbb{E}[\Delta\theta] = -\eta \gamma k \cdot \nabla \mathcal{L}(\theta) + O(\varepsilon^3)
  \end{equation}
\end{theorem}
The above theorem shows that the expected update direction aligns with the negative gradient under coordinate-axis sampling. Furthermore, the analysis of the expected direction under adaptive sampling is provided in \cref{thm:adaptive} (Appendix~\ref{app:adaptive}). 


\subsection{Algorithm}
\label{sec:algorithm}
Clearly, the choice of $\gamma$ is crucial. We observe that the norm of the projected gradient estimated via finite differences remains stable during the early and middle stages of optimization, but tends to grow in later stages due to numerical precision limitations, which restricts the achievable convergence accuracy. To this end, we design a residual-based mechanism that adaptively adjusts $\sigma_e$ after the $\tau$-th sample:
\begin{equation}
  \begin{aligned}
    r_\tau :=& \frac{\hat{y}^{(\tau)} - d^{(\tau)^\top} \mu^{(\tau-1)}}{\|d^{(\tau)}\|},\\
    (\sigma_{e}^{(\tau)})^2 =& (1-\alpha) (\sigma_{e}^{(\tau-1)})^2 + \alpha r_\tau^2,
  \end{aligned}
\end{equation}
where $\alpha\in(0,1)$ is the smoothing factor.

Corollary~\ref{cor:reduced-posterior} shows that under coordinate-axis sampling, the posterior covariance $\Sigma$ degenerates into a diagonal matrix with a single distinct eigenvalue, implying that any axis-aligned direction may serve as the adaptive sampling direction when $j > k$. The residual-based adaptation breaks this degeneracy by differentiating the diagonal entries of $\Sigma$, thereby producing a meaningful adaptive sampling direction. However, the diagonal structure implies that the adaptive sampling direction always coincides with one of the coordinate axes, which can lead to redundant computation. To address this, we cache the $(d, y)$ pairs from the first $k$ samples within each batch. When $j > k$, we directly reuse the cached pair corresponding to the largest diagonal entry of $\Sigma$, eliminating the need for an additional forward pass. This extra sample leverages the updated residual to more precisely correct the step size along the direction of greatest uncertainty. In practice, we set $m = k + 1$ by default.

The main procedure of BSZO is summarized in Algorithm~\ref{alg:bszo}. Following MeZO \cite{malladi2023fine}, we store perturbation vectors via random seeds rather than explicitly, requiring only $O(k^2)$ additional space. A basic version without caching is provided in Algorithm~\ref{alg:bszo-cache} (Appendix~\ref{app:bszo-basic}), which supports arbitrary initial sampling directions and additional adaptive sampling steps. In this version, the adaptive sampling performs extra forward passes to obtain new function values. Typically, the result of this forward pass coincides with the cached value. However, under reduced precision (fp16 or bf16), certain GPU operations use non-deterministic algorithms~\cite{pytorch_reproducibility}, causing function evaluations to differ across calls even with identical inputs and random seeds. Moreover, due to numerical errors, parameters do not fully recover after perturbation and restoration. As a result, the extra forward pass in the basic version yields a value different from the cached one, better reflecting the local landscape at the perturbed point and leading to improved performance (as confirmed in Section~\ref{sec:llm-results}). To examine this effect, we include the coordinate-axis sampling variant of Algorithm~\ref{alg:bszo-cache} as an experimental baseline (denoted as BSZO-B). \cref{tab:memory} compares the memory complexity of different methods, showing that BSZO is also memory-efficient. We analyze the convergence properties of BSZO in the next section.

\begin{table}[tb]
  \caption{Memory complexity comparison of different methods.}
  \label{tab:memory}
  \begin{center}
    \begin{small}
      \begin{tabular}{lcc}
        \toprule
        Method & Memory & Additional Space \\
        \midrule
        MeZO-SGD  & $O(n)$ & $O(1)$ \\
        MeZO-Adam & $O(n)$ & $O(n)$ \\
        HiZOO & $O(n)$ & $O(n)$ \\
        LOZO & $O(n)$ & $O(1)$ \\
        BSZO (Ours) & $O(n)$ & $O(k^2)$ \\
        \bottomrule
      \end{tabular}
    \end{small}
  \end{center}
\end{table}


\section{Theoretical Analysis}
\label{sec:theory}

\begin{table*}[t!]
  \caption{Experiments on three different models (OPT-1.3B, Mistral-7B, OPT-13B). We show the test accuracy (\%) of MeZO, MeZO-Adam, HiZOO, LOZO, BSZO, and BSZO-B on them, with the top two results highlighted in \textbf{bold}. BSZO-B is the baseline version of BSZO. Since fp16 can cause training crashes with Adam, we did not record the results of ZO-Adam for Mistral-7B. * indicates training collapse due to numerical overflow under fp16 precision.}
  \label{tab:llm-results}
  \begin{center}
    \begin{small}
      \begin{sc}
        \begin{tabular*}{\textwidth}{@{\extracolsep{\fill}}llccccccc@{}}
          \toprule
          Model & Method & SST-2 & RTE & COPA & WIC & WSC & TREC & Avg  \\
          \cmidrule(lr){3-3} \cmidrule(lr){4-4} \cmidrule(lr){5-7} \cmidrule(lr){8-8}
          & & sentiment & NLI & \multicolumn{3}{c}{reasoning} & topic & \\
          \midrule
          OPT-1.3B & MeZO      & 91.74 & \textbf{64.98} & 76.0 & 58.78 & 59.62 & 80.6 & 71.95 \\
          OPT-1.3B & MeZO-Adam & \textbf{93.35} & 60.29 & 75.0 & 56.58 & 62.50 & 79.4 & 71.19 \\
          OPT-1.3B & HiZOO     & 91.51 & 62.09 & 77.0 & 56.58 & \textbf{63.46} & 66.2 & 69.48 \\
          OPT-1.3B & LOZO      & 92.66 & 63.18 & 75.0 & 56.58 & 57.69 & 75.8 & 70.15 \\
          OPT-1.3B & BSZO      & 92.43 & \textbf{66.79} & \textbf{79.0} & \textbf{59.88} & \textbf{64.42} & \textbf{87.0} & \textbf{74.92} \\
          OPT-1.3B & BSZO-B    & \textbf{93.01} & \textbf{64.98} & \textbf{81.0} & \textbf{59.09} & 61.54 & \textbf{87.4} & \textbf{74.50} \\
          \midrule
          Mistral-7B & MeZO      & 90.94 & 64.26 & \textbf{88.0} & 56.58 & \textbf{63.46} & 88.6 & 75.31 \\
          Mistral-7B & HiZOO     & 93.01 & 63.90 & \textbf{90.0} & 55.64 & \textbf{63.46} & * & 73.20 \\
          Mistral-7B & LOZO      & 92.43 & 61.37 & 86.0 & 57.83 & 63.46 & * & 72.22 \\
          Mistral-7B & BSZO      & \textbf{94.50} & \textbf{75.81} & 87.0 & \textbf{60.03} & 59.62 & \textbf{90.0} & \textbf{77.83} \\
          Mistral-7B & BSZO-B    & \textbf{94.04} & \textbf{78.70} & 87.0 & \textbf{59.72} & 60.58 & \textbf{91.0} & \textbf{78.51} \\
          \midrule
          OPT-13B & MeZO      & 85.89 & 62.09 & 80.0 & 54.55 & 60.58 & 59.4 & 67.09 \\
          OPT-13B & MeZO-Adam & 79.82 & 61.73 & 81.0 & 54.39 & 57.69 & 62.2 & 66.14 \\
          OPT-13B & HiZOO     & 72.71 & 62.46 & 80.0 & 52.35 & 46.15 & 19.8 & 55.58 \\
          OPT-13B & LOZO      & 86.12 & 57.04 & 80.0 & \textbf{55.96} & 59.62 & 60.4 & 66.52 \\
          OPT-13B & BSZO      & \textbf{93.23} & \textbf{69.31} & \textbf{83.0} & \textbf{56.27} & \textbf{61.54} & \textbf{79.2} & \textbf{73.76} \\
          OPT-13B & BSZO-B    & \textbf{91.86} & \textbf{71.84} & \textbf{85.0} & 53.14 & \textbf{64.42} & \textbf{80.8} & \textbf{74.51} \\
          \bottomrule
        \end{tabular*}
      \end{sc}
    \end{small}
  \end{center}
\end{table*}

\begin{definition}
  \label{def:effective-noise}
  Let $\Sigma = \text{Cov}(\zeta)$ be the covariance matrix of the gradient noise, the effective noise $\sigma_e^2$ can be decomposed as:
  \begin{equation}
    \sigma_e^2 = \sigma_{\varepsilon}^2 + \text{tr}(\Sigma),
  \end{equation}
\end{definition}
where $\text{tr}({\Sigma})\le\sigma_g^2$ (Assumption~\ref{ass:bounded-variance}). The justification for this definition is provided by Lemma~\ref{lem:effective-noise} in Appendix~\ref{app:noise-variance}. For analytical tractability, we assume that $\sigma_e$ is fixed (taking the worst-case noise across batches gives the same result). The convergence of BSZO is characterized by the following theorem:

\begin{theorem}
  \label{thm:main-convergence-theorem}
  Under Assumptions~\ref{ass:l-smooth} and \ref{ass:bounded-variance}, let $\tilde{n}=n+k+1$ be effective dimension. Suppose $m=k$ and $\eta < \frac{2}{L \gamma \tilde{n}}$. Then, after $T$ iterations, the following inequality holds:
   \begin{equation}
    \frac{1}{T}\sum_{t=0}^{T-1} \mathbb{E}[\|\nabla \mathcal{L}(\theta_t)\|^2] \leq \frac{\mathcal{L}(\theta_0) - \mathcal{L}^*}{\beta(\eta) \eta \gamma k T} + \frac{L \eta \gamma (\tilde{n} \cdot \text{tr}(\Sigma) + n \sigma_\varepsilon^2)}{2\beta(\eta)},
   \end{equation}
  where $\beta(\eta) := 1 - \frac{L \eta \gamma \tilde{n}}{2}$ and $\sigma_e^2 = \sigma_{\varepsilon}^2 + \text{tr}(\Sigma)$.
\end{theorem}
\begin{corollary}
  \label{cor:main-convergence-corollary}
  Let $\eta = \frac{1}{L\gamma\tilde{n}}$, then $\beta = 1/2$, which simplifies Theorem~\ref{thm:main-convergence-theorem} to:
   \begin{equation}
    \frac{1}{T}\sum_{t=0}^{T-1} \mathbb{E}[\|\nabla \mathcal{L}(\theta_t)\|^2] \leq \frac{2L\gamma\tilde{n}\Delta_0}{kT} + \text{tr}(\Sigma) + \frac{n}{\tilde{n}}\sigma_{\varepsilon}^2,
   \end{equation}
  where $\Delta_0 := \mathcal{L}(\theta_0) - \mathcal{L}^*$.
\end{corollary}
According to Corollary~\ref{cor:main-convergence-corollary}, the convergence rate of BSZO is improved by the factor of subspace dimension $k$. Although $\gamma$ slightly reduces the convergence rate, it is crucial for training stability. We also analyze the convergence under adaptive sampling in \cref{thm:adaptive-convergence} (Appendix~\ref{app:adaptive-convergence}).

\begin{figure}[t]
  \centering
  \includegraphics[width=\linewidth]{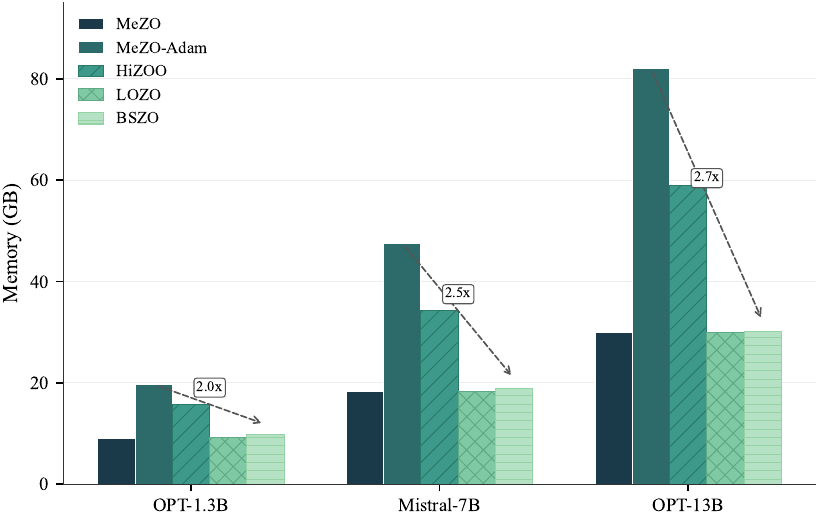}
  \caption{GPU memory usage comparison across different models. BSZO maintains memory consumption comparable to MeZO, while MeZO-Adam and HiZOO require significantly more memory due to storing optimizer states or Hessian estimates.}
  \label{fig:memory}
\end{figure}

\section{Experiments}
\label{sec:experiments}
In this section, we evaluate the performance of BSZO and BSZO-B (Section~\ref{sec:algorithm}) on various fine-tuning tasks in different language models, comparing them with several baselines: MeZO \cite{malladi2023fine}, MeZO-Adam \cite{malladi2023fine}, HiZOO \cite{zhao2025hizoo}, and LOZO \cite{chen2024enhancing}. Our experiments show that both variants achieve excellent robustness and strong accuracy across most scenarios, requiring only the GPU memory needed for forward propagation, making them more cost-effective than HiZOO and MeZO-Adam.
\subsection{Experimental Setup}
\label{sec:experimental-setup}
\textbf{Language Models}. The experiments in this paper center on two categories of models: masked Language Models (mLMs) and decoder-only Large Language Models (LLMs). For mLMs, we adopt RoBERTa-large (355M) \cite{liu2019roberta} as the backbone model. For decoder-only LLMs, we select OPT-1.3B and OPT-13B \cite{zhang2022opt}, as well as Mistral-7B \cite{jiang2023mistral}. \\

\textbf{Datasets}. We full fine-tune the above models on tasks from the GLUE \cite{wang2018glue}, SuperGLUE \cite{wang2019superglue} and TREC \cite{li2002learning} benchmarks, including Stanford Sentiment Treebank (SST-2), Boolean Questions (BoolQ) \cite{clark2019boolq}, Recognizing Textual Entailment (RTE) \cite{dagan2005pascal}, Choice of Plausible Alternatives (COPA) \cite{roemmele2011choice}, Word-in-Context (WIC) \cite{pilehvar2019wic}, Winograd Schema Challenge (WSC) \cite{levesque2012winograd}, CommitmentBank (CB) \cite{demarneffe2019commitmentbank}, and TREC. Following HiZOO \cite{zhao2025hizoo}, we use the first $n_1$ samples (up to 1000) from the training set for training and the next $n_2$ samples for validation. The original validation set serves as the test set. See Table~\ref{tab:dataset-splits} in Appendix~\ref{app:experiment-details} for specific values of $n_1$ and $n_2$.\\

\textbf{Hyperparameters}. For BSZO and BSZO-B, we set the default subspace dimension $k=2$ and the number of samples $m=k + 1$. This results in 3 forward passes per step for BSZO (with caching) and 4 for BSZO-B (without caching). BSZO matches HiZOO's forward pass count. As discussed in Section~\ref{sec:algorithm}, we report results for both BSZO and BSZO-B across all models, with particular focus on comparing them under reduced precision (Mistral-7B in fp16 and OPT-13B in bf16) to examine the caching effect. Other methods use their default hyperparameters. Given the slower convergence of zeroth-order methods, all experiments are trained for up to 20,000 steps \cite{zhang2024revisiting}, with early stopping applied when validation performance does not improve for 8 evaluations (4,000 steps). For every experiment, we set the perturbation scale to $\varepsilon=10^{-4}$ and the batch size to 16. Hyperparameters are tuned via grid search. We select the best configuration based on validation performance and report its test accuracy. Due to memory constraints, we load OPT-13B in bf16 precision and Mistral-7B in fp16 precision, while other models use fp32. All experiments are conducted on a single H200 GPU. More details are provided in Appendix~\ref{app:experiment-details}.

\subsection{Performance in Masked Language Models}
\label{sec:mlm-results}
\textbf{BSZO achieves stable and competitive performance on mLMs.} As shown in Table~\ref{tab:roberta}, BSZO-B reaches 77.39\% average accuracy on RoBERTa-large, surpassing MeZO (77.04\%, +0.35\%), MeZO-Adam (73.73\%, +3.66\%), HiZOO (68.82\%, +8.57\%), and LOZO (74.14\%, +3.25\%). BSZO achieves 77.28\% average accuracy, second only to BSZO-B. BSZO secures top result on SST-2 (92.66\%), while BSZO-B excels on RTE (68.38\%) and WIC (57.21\%). Moreover, BSZO exhibits notably lower variance across tasks (see Table~\ref{tab:roberta-raw} for raw results). Both variants demonstrate strong and consistent performance across all tasks.

\subsection{Performance in decoder-only models}
\label{sec:llm-results}
\textbf{BSZO performs well on larger LLMs.} Table~\ref{tab:llm-results} shows that BSZO outperforms baselines on decoder-only models, with gains increasing as model size grows. BSZO-B typically maintains a small lead over BSZO.\\
\textbf{OPT-1.3B.} BSZO achieves 74.92\% average accuracy, the highest among all methods, beating MeZO (71.95\%, +2.97\%), MeZO-Adam (71.19\%, +3.73\%), HiZOO (69.48\%, +5.44\%), and LOZO (70.15\%, +4.77\%). BSZO-B reaches 74.50\% average accuracy. BSZO secures top results on RTE (66.79\%), WIC (59.88\%), and WSC (64.42\%), while BSZO-B excels on COPA (81.0\%) and TREC (87.4\%). Both variants perform well across most tasks. \\
\textbf{Mistral-7B (fp16).} BSZO reaches 77.83\% on average, ahead of MeZO (75.31\%, +2.52\%), HiZOO (73.20\%, +4.63\%), and LOZO (72.22\%, +5.61\%). It also achieves the best results on SST-2 (94.50\%, +1.49\% vs HiZOO) and WIC (60.03\%, +3.45\% vs MeZO). BSZO-B reaches 78.51\% on average, excelling on RTE (78.70\%) and TREC (91.0\%). The small 0.68\% gap shows that the two variants perform very similarly.\\
\textbf{OPT-13B (bf16).} The gains grow larger here. BSZO reaches 73.76\% on average, up 6.67\% over MeZO (67.09\%), 7.62\% over MeZO-Adam (66.14\%), and 7.24\% over LOZO (66.52\%). BSZO achieves strong results across tasks, including top performance on WIC (56.27\%), with particularly notable gains on SST-2 (93.23\%, +7.34\% vs MeZO) and TREC (79.2\%, +19.8\% vs MeZO). BSZO-B reaches 74.51\% on average (+7.42\% vs MeZO), with stronger balance across tasks. BSZO-B maintains a slight edge with one additional forward pass, though the gap in average accuracy remains very small (0.75\%).\\
\textbf{Robustness.} Reduced precision exposes fragility in several baselines (Table~\ref{tab:llm-results}) and Figure~\ref{fig:motivation}. HiZOO and LOZO are particularly affected: on Mistral-7B (fp16), both methods suffer from TREC training overflow (*). On OPT-13B (bf16), all baseline methods show varying degrees of performance degradation compared to OPT-1.3B, with HiZOO being especially severe—its average accuracy drops from 69.48\% to 55.58\%, with TREC collapsing to 19.8\% and WSC to 46.15\%. We suspect H200's default TF32 mode introduces errors in Hessian-based estimates. In contrast, BSZO and BSZO-B remain stable throughout all precision settings, with BSZO-B even maintaining performance (from 74.50\% to 74.51\%).

\subsection{Memory and Time Efficiency}
\label{sec:efficiency}

\begin{table}[t]
  \caption{Memory usage (GB) and per-step time (ms) across different models.}
  \label{tab:efficiency}
  \begin{center}
    \begin{small}
      \setlength{\tabcolsep}{3pt}
      \begin{tabular}{l cc cc cc}
        \toprule
        & \multicolumn{2}{c}{OPT-1.3B} & \multicolumn{2}{c}{Mistral-7B} & \multicolumn{2}{c}{OPT-13B} \\
        \cmidrule(lr){2-3} \cmidrule(lr){4-5} \cmidrule(lr){6-7}
        Method & Mem & Time & Mem & Time & Mem & Time \\
        \midrule
        MeZO      & 9.1  & 109.7 & 18.3 & 283.9 & 30.0 & 464.2 \\
        MeZO-Adam & 19.7 & 135.1 & 47.6 & 373.1 & 82.1 & 614.5 \\
        HiZOO     & 15.7 & 188.0 & 34.3 & 540.2 & 58.9 & 877.1 \\
        LOZO      & 9.3  & 102.0 & 18.3 & 274.2 & 30.0 & 452.0 \\
        BSZO      & 9.8  & 97.0  & 18.8 & 275.7 & 30.1 & 440.5 \\
        \bottomrule
      \end{tabular}
    \end{small}
  \end{center}
\end{table}

\textbf{BSZO keeps memory usage low.} As shown in Figure~\ref{fig:memory} and Table~\ref{tab:efficiency}, BSZO's memory footprint stays close to MeZO across three model scales—ranging from 1.00$\times$ to 1.08$\times$ of MeZO's usage. In contrast, HiZOO and MeZO-Adam need 1.73$\times$--1.96$\times$ and 2.16$\times$--2.74$\times$ more memory because they store additional optimizer states (momentum, Hessian estimates). BSZO avoids this overhead by using only $O(k^2)$ extra space for the posterior covariance and adaptive noise estimation.

\textbf{BSZO runs fast.} Table~\ref{tab:efficiency} also reports per-step time. BSZO and LOZO are the fastest—both under 100ms per step on OPT-1.3B. HiZOO is roughly 2$\times$ slower due to Hessian estimation, and MeZO-Adam incurs extra cost from momentum updates.

\subsection{Ablation Study}
\label{sec:ablation}

\begin{table}[!t]
  \caption{Ablation studies. (a) Effect of subspace dimension $k$ with $m=k$ on OPT-1.3B. (b) Effect of $m=k+1$ on OPT-1.3B. (c) Effect of adaptive noise on OPT-1.3B (bf16). (d) Effect of adaptive noise on OPT-13B (bf16). Best results in \textbf{bold}. Full results in fp32 are in Table~\ref{tab:ablation-full}.}
  \label{tab:ablation}
  \begin{center}
    \begin{small}
      \setlength{\tabcolsep}{4pt}
      \begin{tabular}{@{}ccc@{\hspace{8pt}}ccc@{\hspace{8pt}}ccc@{}}
        \toprule
        \multicolumn{3}{c}{\textbf{(a) Effect of $k$}} & \multicolumn{3}{c}{\textbf{(b) Effect of $m$}} & \multicolumn{3}{c}{\textbf{(c) Adaptive Noise}} \\
        \cmidrule(r){1-3} \cmidrule(lr){4-6} \cmidrule(l){7-9}
        $k$ & SST-2 & RTE & $k$ & SST-2 & RTE & $k$ & w/ & w/o \\
        \midrule
        1 & \textbf{92.32} & 60.29 & 1 & 91.74 & 61.37 & 1 & 54.15 & \textbf{55.24} \\
        2 & \textbf{92.78} & 64.26 & 2 & 92.43 & \textbf{66.79} & 2 & \textbf{57.76} & 56.32 \\
        4 & 92.66 & \textbf{67.51} & 4 & \textbf{93.58} & 66.43 & 4 & \textbf{61.73} & 56.32 \\
        8 & 93.23 & 66.07 & 8 & 93.23 & 68.59 & 8 & \textbf{66.43} & 57.76 \\
        \bottomrule
      \end{tabular}
      \vspace{6pt}

      \begin{tabular}{@{}lcccccc@{}}
        \toprule
        \multicolumn{7}{c}{\textbf{(d) Adaptive Noise on OPT-13B (bf16)}} \\
        \cmidrule{1-7}
        Method & SST-2 & RTE & WSC & COPA & TREC & WIC \\
        \midrule
        w/ adaptive & \textbf{93.23} & \textbf{69.31} & \textbf{61.54} & 83.00 & \textbf{79.20} & \textbf{56.27} \\
        w/o adaptive & 91.97 & 63.18 & 58.65 & \textbf{85.00} & 75.00 & 54.70 \\
        \bottomrule
      \end{tabular}
    \end{small}
  \end{center}
\end{table}

Table~\ref{tab:ablation} shows ablation results on OPT-1.3B for two design choices of BSZO: subspace dimension $k$ and sample count $m$. In Table~\ref{tab:ablation}(a), when $m=k$, RTE accuracy climbs from 60.29\% ($k=1$) to 67.51\% ($k=4$), while SST-2 peaks at $k=8$ (93.23\%), suggesting that increasing $k$ generally improves performance. In Table~\ref{tab:ablation}(b), with extra refinement ($m=k+1$), RTE performance improves consistently. Comparing to Table~\ref{tab:ablation}(a), $m=k+1$ boosts RTE by 1-2\% at most $k$ levels (e.g., from 64.26\% to 66.79\% at $k=2$, from 66.07\% to 68.59\% at $k=8$). This confirms that the adaptive sampling step refines the posterior estimate (see Table~\ref{tab:ablation-full} for more details). Table~\ref{tab:ablation}(c) investigates adaptive noise under bf16 precision on OPT-1.3B. As $k$ grows, the gap between w/ and w/o adaptive noise becomes more pronounced: at $k=8$, the adaptive variant leads by 8.67\% on RTE, indicating that adaptive noise yields substantial gains in low-precision settings. Table~\ref{tab:ablation}(d) validates this on OPT-13B (bf16), where adaptive noise brings improvements on 5 out of 6 tasks, with RTE gaining 6.13\%.


\section{Conclusion}

In this work, we introduce BSZO, which is the first zeroth-order optimizer that applies Kalman filtering to aggregate gradient information across multiple perturbation directions for LLM fine-tuning. By treating finite-difference measurements as noisy observations of the true gradient, BSZO builds a posterior distribution over the projected gradient and refines it through Bayesian updates. We design a residual-based adaptive mechanism to adjust the perturbation scale adaptively without manual tuning. Our theoretical analysis shows that BSZO improves the convergence rate by a factor of $k/\gamma$ over standard ZO methods. Experiments on RoBERTa, Mistral, and OPT show that BSZO achieves strong accuracy across various tasks, remains stable under fp16/bf16 precision where existing methods often collapse, and keeps memory usage close to inference-only baselines.

\section*{Software and Data}

Our implementation is available at \url{https://github.com/AeonianQuill/BSZO}. Datasets used in this work (GLUE, SuperGLUE, TREC) are publicly accessible and should be downloaded separately. Pre-trained models can also be obtained from Hugging Face. 

\section*{Impact Statement}

This paper presents work whose goal is to advance the field of machine learning. There are many potential societal consequences of our work, none of which we feel must be specifically highlighted here.

\bibliography{my_paper}
\bibliographystyle{icml2026}

\newpage
\appendix
\onecolumn

\section{BSZO Basic Algorithm}
\label{app:bszo-basic}

We provide a basic version of BSZO without caching optimization, which supports arbitrary sampling directions when $m > k$.

\begin{algorithm}[H]
  \caption{Bayesian Subspace Zeroth-Order Optimization (Basic Version)}
  \label{alg:bszo-cache}
  \begin{algorithmic}
    \STATE {\bfseries Input:} parameters $\theta$, learning rate $\eta$, perturbation scale $\varepsilon$, subspace dimension $k$, sampling steps $m$, prior variance $\sigma_p^2$, noise variance $\sigma_e^2$, smoothing factor $\alpha$, max step $T$
    \FOR{$t=1$ {\bfseries to} $T$}
    \STATE Sample $k$ random seeds $\{s_i\}_{i=1}^k$
    \STATE Initialize $\mu \leftarrow \mathbf{0}_k$, $\Sigma \leftarrow \sigma_p^2 I_k$, $f_0 \leftarrow \mathcal{L}(\theta)$
    \FOR{$\tau=1$ {\bfseries to} $m$}
      \STATE $d \leftarrow d_\tau$ if $\tau \leq k$, else $d \leftarrow \arg\max_{\|v\|=1} v^\top \Sigma v$
      \FOR{$i=1$ {\bfseries to} $k$}
        \STATE $\theta \leftarrow \theta + \varepsilon \cdot d_i \cdot \textsc{Randn}(n, s_i)$ if $d_i > 10^{-10}$
      \ENDFOR
      \STATE $y \leftarrow (\mathcal{L}(\theta) - f_0) / \varepsilon$
      \FOR{$i=1$ {\bfseries to} $k$}
        \STATE $\theta \leftarrow \theta - \varepsilon \cdot d_i \cdot \textsc{Randn}(n, s_i)$ if $d_i > 10^{-10}$
      \ENDFOR
      \STATE $r \leftarrow (y - d^\top \mu) / \|d\|$, \quad $\sigma_e^2 \leftarrow (1-\alpha)\sigma_e^2 + \alpha r^2$
      \STATE $K \leftarrow \Sigma d / (d^\top \Sigma d + \sigma_e^2)$
      \STATE $\mu \leftarrow \mu + K (y - d^\top \mu)$, \quad $\Sigma \leftarrow \Sigma - K d^\top \Sigma$
    \ENDFOR
    \FOR{$i=1$ {\bfseries to} $k$}
      \STATE $\theta \leftarrow \theta - \eta \cdot \mu_i \cdot \textsc{Randn}(n, s_i)$
    \ENDFOR
  \ENDFOR
  \STATE {\bfseries return} $\theta$
  \STATE {\small $\textsc{Randn}(n,s)$: returns $n$-dim Gaussian vector seeded by $s$}
  \end{algorithmic}
\end{algorithm}

\section{Theoretical Proofs}
\label{app:proofs}

\subsection{Auxiliary Lemmas}
\label{app:auxiliary}

\begin{lemma}[Expectation of Direction Derivative]
\label{lem:direction-derivative-expectation}
Under Assumption~\ref{ass:l-smooth}, the one-sided difference satisfies:
\begin{equation}
  \mathbb{E}[\hat{y}(d)] = d^\top B^\top g + O(\varepsilon L)
\end{equation}
\end{lemma}
\begin{proof}
By $\mathbb{E}[\mathcal{L}(\theta;\xi)] = \mathcal{L}(\theta)$:
\begin{equation}
  \mathbb{E}[\hat{y}(d)] = \frac{\mathcal{L}(\theta_0 + \varepsilon Bd) - \mathcal{L}(\theta_0)}{\varepsilon}
\end{equation}

By Taylor expansion at $\theta_0$:
\begin{equation}
  \mathcal{L}(\theta_0 + \varepsilon Bd) = \mathcal{L}(\theta_0) + \varepsilon \langle \nabla \mathcal{L}(\theta_0), Bd \rangle + \frac{\varepsilon^2}{2} (Bd)^\top H (Bd) + O(\varepsilon^3)
\end{equation}
where $H = \nabla^2 \mathcal{L}(\theta_0)$ is the Hessian.

Substituting:
\begin{equation}
  \mathbb{E}[\hat{y}(d)] = \frac{\varepsilon d^\top B^\top g + \frac{\varepsilon^2}{2} d^\top B^\top H B d + O(\varepsilon^3)}{\varepsilon} = d^\top B^\top g + \frac{\varepsilon}{2} d^\top B^\top H B d + O(\varepsilon^2)
\end{equation}

By Assumption~\ref{ass:l-smooth}, $\|H\| \leq L$, so $|d^\top B^\top H B d| \leq L\|Bd\|^2$. Thus the bias is $O(\varepsilon L)$.
\end{proof}

\begin{lemma}[Variance of Direction Derivative]
\label{lem:direction-derivative-variance}
Let $\Sigma = \text{Cov}(\zeta)$ where $\zeta = \nabla \mathcal{L}(\theta;\xi) - \nabla \mathcal{L}(\theta)$. When using the same mini-batch $\xi$ for both function evaluations:
\begin{itemize}
  \item[(a)] Conditional variance: $\text{Var}(\hat{y}(d)|B) = (Bd)^\top \Sigma (Bd) + O(\varepsilon^4)$
  \item[(b)] Unconditional variance: $\mathbb{E}_B[\text{Var}(\hat{y}(d)|B)] = \text{tr}(\Sigma) + O(\varepsilon^4)$
\end{itemize}
\end{lemma}
\begin{proof}
\textbf{Key insight}: Using the \textbf{same} mini-batch $\xi$ for both evaluations causes the noise to be correlated, not independent.

\textbf{(a) Conditional variance derivation:}

For fixed $\xi$, Taylor expand the random loss $\mathcal{L}(\theta;\xi)$ at $\theta_0$:
\begin{equation}
  \mathcal{L}(\theta_0 + \varepsilon Bd; \xi) = \mathcal{L}(\theta_0; \xi) + \varepsilon \langle \nabla \mathcal{L}(\theta_0; \xi), Bd \rangle + O(\varepsilon^2)
\end{equation}

Since both evaluations use the \textbf{same} $\xi$, the base term $\mathcal{L}(\theta_0; \xi)$ cancels:
\begin{equation}
  \hat{y}(d) = \frac{\mathcal{L}(\theta_0 + \varepsilon Bd; \xi) - \mathcal{L}(\theta_0; \xi)}{\varepsilon} = \langle \nabla \mathcal{L}(\theta_0; \xi), Bd \rangle + O(\varepsilon)
\end{equation}

Let $\nabla \mathcal{L}(\theta_0; \xi) = \nabla \mathcal{L}(\theta_0) + \zeta$ where $\zeta$ is zero-mean noise with $\text{Cov}(\zeta) = \Sigma$. Given $B$:
\begin{equation}
  \text{Var}(\hat{y}(d)|B) = \text{Var}(\langle \zeta, Bd \rangle | B) = (Bd)^\top \text{Cov}(\zeta) (Bd) = (Bd)^\top \Sigma (Bd) + O(\varepsilon^2)
\end{equation}

\textbf{(b) Unconditional variance derivation:}

For coordinate-axis sampling $d = e_i$, we have $Bd = z_i \sim \mathcal{N}(0, I_n)$.

Taking expectation over $B$:
\begin{equation}
  \mathbb{E}_B[(Bd)^\top \Sigma (Bd)] = \mathbb{E}[z_i^\top \Sigma z_i]
\end{equation}

By the trace trick:
\begin{equation}
  \mathbb{E}[z_i^\top \Sigma z_i] = \mathbb{E}[\text{tr}(\Sigma z_i z_i^\top)] = \text{tr}(\Sigma \cdot \mathbb{E}[z_i z_i^\top]) = \text{tr}(\Sigma \cdot I_n) = \text{tr}(\Sigma)
\end{equation}

By Assumption~\ref{ass:bounded-variance}, $\text{tr}(\Sigma) = \mathbb{E}[\|\zeta\|^2] \leq \sigma_g^2$.
\end{proof}

\begin{lemma}[High-Dimensional Approximate Orthogonality]
\label{lem:high-dim-orthogonality}
Let $z_1, \ldots, z_k \stackrel{iid}{\sim} \mathcal{N}(0, I_n)$. When $n \gg k^2$:
\begin{itemize}
  \item[(a)] $\|z_i\|^2 = n \pm O(\sqrt{n})$
  \item[(b)] For $i \neq j$: $\frac{z_i^\top z_j}{\|z_i\|\|z_j\|} = O(1/\sqrt{n})$
\end{itemize}
\end{lemma}
\begin{proof}
\textbf{(a) Norm concentration:}

Since $\|z_i\|^2 = \sum_{j=1}^n z_{ij}^2 \sim \chi^2(n)$, we have:
\begin{equation}
  \mathbb{E}[\|z_i\|^2] = n, \quad \text{Var}(\|z_i\|^2) = 2n
\end{equation}

By Chebyshev inequality or sub-Gaussian concentration:
\begin{equation}
  \mathbb{P}\left(\left|\|z_i\|^2 - n\right| > t\sqrt{n}\right) \leq 2e^{-ct^2}
\end{equation}
Thus $\|z_i\|^2 = n \pm O(\sqrt{n})$ with high probability.

\textbf{(b) Approximate orthogonality:}

For independent $z_i, z_j \sim \mathcal{N}(0, I_n)$, the inner product $z_i^\top z_j = \sum_{l=1}^n z_{il} z_{jl}$ is a sum of $n$ independent random variables with:
\begin{equation}
  \mathbb{E}[z_i^\top z_j] = 0, \quad \text{Var}(z_i^\top z_j) = \sum_{l=1}^n \text{Var}(z_{il} z_{jl}) = n
\end{equation}

Thus $z_i^\top z_j = O(\sqrt{n})$ with high probability. Since $\|z_i\|\|z_j\| = O(n)$:
\begin{equation}
  \cos\theta_{ij} = \frac{z_i^\top z_j}{\|z_i\|\|z_j\|} = O\left(\frac{\sqrt{n}}{n}\right) = O\left(\frac{1}{\sqrt{n}}\right) \to 0
\end{equation}

This shows that random Gaussian vectors are approximately orthogonal in high dimensions.
\end{proof}

\begin{lemma}[Isserlis' Theorem Application]
\label{lem:isserlis}
For $z \sim \mathcal{N}(0, I_n)$ and symmetric matrices $A, B$:
\begin{equation}
  \mathbb{E}[(z^\top A z)(z^\top B z)] = \text{tr}(A)\text{tr}(B) + 2\text{tr}(AB)
\end{equation}
In particular, for $A = I_n$ and $B = \Sigma$:
\begin{equation}
  \mathbb{E}[\|z\|^2 \cdot z^\top \Sigma z] = (n+2)\text{tr}(\Sigma)
\end{equation}
\end{lemma}
\begin{proof}
By Isserlis' theorem (Wick's theorem), for $z \sim \mathcal{N}(0, I_n)$:
\begin{equation}
  \mathbb{E}[z_i z_j z_k z_l] = \delta_{ij}\delta_{kl} + \delta_{ik}\delta_{jl} + \delta_{il}\delta_{jk}
\end{equation}

Expanding the quadratic forms:
\begin{equation}
  (z^\top A z)(z^\top B z) = \sum_{i,j,k,l} A_{ij} B_{kl} z_i z_j z_k z_l
\end{equation}

Taking expectation:
\begin{align}
  \mathbb{E}[(z^\top A z)(z^\top B z)] &= \sum_{i,j,k,l} A_{ij} B_{kl} (\delta_{ij}\delta_{kl} + \delta_{ik}\delta_{jl} + \delta_{il}\delta_{jk}) \\
  &= \sum_{i,k} A_{ii} B_{kk} + \sum_{i,j} A_{ij} B_{ij} + \sum_{i,j} A_{ij} B_{ji} \\
  &= \text{tr}(A)\text{tr}(B) + \text{tr}(AB) + \text{tr}(AB^\top)
\end{align}

For symmetric $A, B$: $\text{tr}(AB^\top) = \text{tr}(AB)$, thus:
\begin{equation}
  \mathbb{E}[(z^\top A z)(z^\top B z)] = \text{tr}(A)\text{tr}(B) + 2\text{tr}(AB)
\end{equation}

Setting $A = I_n$, $B = \Sigma$:
\begin{equation}
  \mathbb{E}[\|z\|^2 \cdot z^\top \Sigma z] = n \cdot \text{tr}(\Sigma) + 2\text{tr}(\Sigma) = (n+2)\text{tr}(\Sigma)
\end{equation}
\end{proof}

\subsection{Noise Variance Justification}
\label{app:noise-variance}

The observation model $\hat{y}(d) = d^\top \tilde{g} + \nu$ with $\nu \sim \mathcal{N}(0, \sigma_e^2\|d\|^2)$ is justified as follows.

\begin{lemma}[Effective Noise Decomposition]
\label{lem:effective-noise}
The effective noise variance decomposes as:
\begin{equation}
  \sigma_e^2 = \sigma_\varepsilon^2 + \text{tr}(\Sigma)
\end{equation}
where $\sigma_\varepsilon^2$ is the finite-difference approximation error and $\text{tr}(\Sigma)$ is the gradient noise variance.
\end{lemma}
\begin{proof}
\textbf{Step 1: Decomposition of the observation.}

For coordinate-axis sampling with $d = e_i$, the direction in parameter space is $z_i = Bd = Be_i$ (the $i$-th column of $B$), where $z_i \sim \mathcal{N}(0, I_n)$.

The observation can be decomposed as:
\begin{equation}
  y_i = \underbrace{z_i^\top g}_{\text{true signal}} + \underbrace{z_i^\top \zeta}_{\text{gradient noise}} + \underbrace{\epsilon_i}_{\text{finite-diff error}}
\end{equation}
where:
\begin{itemize}
  \item $g = \nabla \mathcal{L}(\theta)$ is the true gradient
  \item $\zeta = \nabla \mathcal{L}(\theta;\xi) - \nabla \mathcal{L}(\theta)$ is the stochastic gradient noise with $\mathbb{E}[\zeta] = 0$ and $\text{Cov}(\zeta) = \Sigma$
  \item $\epsilon_i \sim \mathcal{N}(0, \sigma_\varepsilon^2)$ is the finite-difference truncation error, independent of $\zeta$ and $z_i$
\end{itemize}

\textbf{Step 2: Identifying the noise term.}

The observation noise is defined as $\nu_i := y_i - z_i^\top g = z_i^\top \zeta + \epsilon_i$.

Since $\mathbb{E}[\zeta] = 0$ and $\mathbb{E}[\epsilon_i] = 0$, we have $\mathbb{E}[\nu_i | z_i] = 0$.

\textbf{Step 3: Conditional variance (given $z_i$).}

Since $\zeta$ and $\epsilon_i$ are independent:
\begin{equation}
  \text{Var}(\nu_i | z_i) = \text{Var}(z_i^\top \zeta | z_i) + \text{Var}(\epsilon_i) = z_i^\top \Sigma z_i + \sigma_\varepsilon^2
\end{equation}

\textbf{Step 4: Unconditional variance (taking expectation over $z_i$).}

Using the trace trick from Lemma~\ref{lem:direction-derivative-variance}(b):
\begin{equation}
  \mathbb{E}_{z_i}[z_i^\top \Sigma z_i] = \mathbb{E}[\text{tr}(\Sigma z_i z_i^\top)] = \text{tr}(\Sigma \cdot \mathbb{E}[z_i z_i^\top]) = \text{tr}(\Sigma \cdot I_n) = \text{tr}(\Sigma)
\end{equation}

Therefore, the effective noise variance is:
\begin{equation}
  \sigma_e^2 := \mathbb{E}_{z_i}[\text{Var}(\nu_i|z_i)] = \text{tr}(\Sigma) + \sigma_\varepsilon^2
\end{equation}

By Assumption~\ref{ass:bounded-variance}, $\text{tr}(\Sigma) = \mathbb{E}[\|\zeta\|^2] \leq \sigma_g^2$, so $\sigma_e^2 \leq \sigma_g^2 + \sigma_\varepsilon^2$.
\end{proof}

\subsection{Proof of Main Convergence Theorem}
\label{app:main-proof}

\begin{proof}[Proof of Theorem~\ref{thm:main-convergence-theorem}]

\textbf{Step 1: Single-step descent.}

By Assumption~\ref{ass:l-smooth} ($L$-smoothness):
\begin{equation}
  \mathcal{L}(\theta_{t+1}) \leq \mathcal{L}(\theta_t) + \langle g_t, \Delta\theta_t \rangle + \frac{L}{2}\|\Delta\theta_t\|^2
\end{equation}
where $g_t = \nabla \mathcal{L}(\theta_t)$ and $\Delta\theta_t = -\eta B_t \mu_t^{(k)}$.

\textbf{Step 2: Inner product term.}

By Lemma~\ref{lem:direction-derivative-variance}, the observation model is $y_i = z_i^\top g_t + z_i^\top \zeta + \epsilon_i$, where $\zeta$ is gradient noise with $\text{Cov}(\zeta) = \Sigma$, and $\epsilon_i \sim \mathcal{N}(0, \sigma_\varepsilon^2)$ is finite-difference error.

Let $\tilde{g} = g_t + \zeta$ and $\epsilon = [\epsilon_1, \ldots, \epsilon_k]^\top$. Then:
\begin{equation}
  Y = B_t^\top \tilde{g} + \epsilon, \quad \mu_t^{(k)} = \gamma Y = \gamma(B_t^\top \tilde{g} + \epsilon)
\end{equation}

The parameter update becomes:
\begin{equation}
  B_t \mu_t^{(k)} = \gamma B_t B_t^\top(g_t + \zeta) + \gamma \sum_{i=1}^k z_i \epsilon_i
\end{equation}

Computing the expectation of the inner product. Since $\mathbb{E}[\zeta] = 0$, $\mathbb{E}[\epsilon_i] = 0$, and $\zeta$, $\epsilon_i$ are independent of $B_t$:
\begin{equation}
  \mathbb{E}[\langle g_t, B_t\mu_t^{(k)} \rangle | \theta_t] = \gamma\mathbb{E}[g_t^\top B_t B_t^\top g_t] + \gamma\mathbb{E}[g_t^\top B_t B_t^\top \zeta] + \gamma\sum_{i=1}^k \mathbb{E}[(g_t^\top z_i)\epsilon_i]
\end{equation}

The second term: $\mathbb{E}[g_t^\top B_t B_t^\top \zeta] = g_t^\top \mathbb{E}[B_t B_t^\top]\mathbb{E}[\zeta] = 0$.

The third term: $\mathbb{E}[(g_t^\top z_i)\epsilon_i] = \mathbb{E}[g_t^\top z_i]\mathbb{E}[\epsilon_i] = 0$ (independence).

The first term: $\mathbb{E}[g_t^\top B_t B_t^\top g_t] = \mathbb{E}[\|B_t^\top g_t\|^2] = \sum_{i=1}^k \mathbb{E}[(z_i^\top g_t)^2] = k\|g_t\|^2$.

Therefore:
\begin{equation}
  \mathbb{E}[\langle g_t, \Delta\theta_t \rangle | \theta_t] = -\eta\gamma k\|g_t\|^2
\end{equation}

\textbf{Step 3: Second moment (detailed computation).}

From Step 2, $B_t \mu_t^{(k)} = \gamma B_t B_t^\top \tilde{g} + \gamma \sum_i z_i \epsilon_i$. Thus:
\begin{equation}
  \|B_t \mu_t^{(k)}\|^2 = \gamma^2\|B_t B_t^\top \tilde{g}\|^2 + \gamma^2\left\|\sum_i z_i\epsilon_i\right\|^2 + 2\gamma^2\left\langle B_t B_t^\top \tilde{g}, \sum_i z_i\epsilon_i \right\rangle
\end{equation}

\textit{Cross term vanishes}: Since $\epsilon_i$ is independent of $B_t$ and $\tilde{g}$, and $\mathbb{E}[\epsilon_i] = 0$:
\begin{equation}
  \mathbb{E}\left[\left\langle B_t B_t^\top \tilde{g}, \sum_i z_i\epsilon_i \right\rangle\right] = \sum_i \mathbb{E}[\epsilon_i] \cdot \mathbb{E}[\langle B_t B_t^\top \tilde{g}, z_i \rangle] = 0
\end{equation}

\textit{First term}: We compute $\mathbb{E}[\|B_t B_t^\top \tilde{g}\|^2]$ by first conditioning on $B_t$, then taking expectation over $B_t$.

(A) Given $B_t$, taking expectation over $\zeta$ (using $\mathbb{E}[\zeta] = 0$):
\begin{equation}
  \mathbb{E}[\|B_t B_t^\top \tilde{g}\|^2 | B_t] = \|B_t B_t^\top g_t\|^2 + \mathbb{E}[\|B_t B_t^\top \zeta\|^2 | B_t]
\end{equation}
where $\mathbb{E}[\|B_t B_t^\top \zeta\|^2 | B_t] = \text{tr}((B_t B_t^\top)^2 \Sigma)$.

(B) Taking expectation over $B_t$. For $\mathbb{E}[\|B_t B_t^\top g\|^2]$:
\begin{equation}
  \|B_t B_t^\top g\|^2 = \sum_{i,j=1}^k (z_i^\top g)(z_j^\top g)(z_i^\top z_j)
\end{equation}

Diagonal terms ($i = j$): $\mathbb{E}[(z_i^\top g)^2 \|z_i\|^2] = (n+2)\|g\|^2$ (by Lemma~\ref{lem:isserlis}).

Off-diagonal terms ($i \neq j$): By independence of $z_i$ and $z_j$:
\begin{align}
  \mathbb{E}[(z_i^\top g)(z_j^\top g)(z_i^\top z_j)] &= \sum_{a,b,c} g_a g_b \mathbb{E}[(z_i)_a (z_i)_c] \mathbb{E}[(z_j)_b (z_j)_c] \\
  &= \sum_{a,b,c} g_a g_b \delta_{ac} \delta_{bc} = \sum_c g_c^2 = \|g\|^2
\end{align}

Thus:
\begin{equation}
  \mathbb{E}[\|B_t B_t^\top g\|^2] = k(n+2)\|g\|^2 + k(k-1)\|g\|^2 = k(n+k+1)\|g\|^2 = k\tilde{n}\|g\|^2
\end{equation}

For $\mathbb{E}[\text{tr}((B_t B_t^\top)^2 \Sigma)]$, let $P = B_t B_t^\top$:
\begin{equation}
  \text{tr}(P^2 \Sigma) = \sum_{i,j} (z_i^\top z_j)(z_i^\top \Sigma z_j)
\end{equation}

Diagonal ($i=j$): By Lemma~\ref{lem:isserlis}, $\mathbb{E}[\|z\|^2 \cdot z^\top \Sigma z] = (n+2)\text{tr}(\Sigma)$.

Off-diagonal ($i \neq j$): By independence of $z_i$ and $z_j$:
\begin{align}
  \mathbb{E}[(z_i^\top z_j)(z_i^\top \Sigma z_j)] &= \sum_{a,b,c} \Sigma_{bc} \mathbb{E}[(z_i)_a (z_i)_b] \mathbb{E}[(z_j)_a (z_j)_c] \\
  &= \sum_{a,b,c} \Sigma_{bc} \delta_{ab} \delta_{ac} = \sum_a \Sigma_{aa} = \text{tr}(\Sigma)
\end{align}

Thus:
\begin{equation}
  \mathbb{E}[\text{tr}(P^2 \Sigma)] = k(n+2)\text{tr}(\Sigma) + k(k-1)\text{tr}(\Sigma) = k\tilde{n} \cdot \text{tr}(\Sigma)
\end{equation}

\textit{Second term}: For the finite-difference noise:
\begin{equation}
  \mathbb{E}\left[\left\|\sum_{i=1}^k z_i \epsilon_i\right\|^2\right] = \sum_{i,j} \mathbb{E}[\epsilon_i \epsilon_j] \mathbb{E}[z_i^\top z_j] = \sum_i \sigma_\varepsilon^2 \cdot n = kn\sigma_\varepsilon^2
\end{equation}

\textit{Total second moment}:
\begin{equation}
  \mathbb{E}[\|B_t\mu_t^{(k)}\|^2] = \gamma^2 k\left(\tilde{n}(\|g_t\|^2 + \text{tr}(\Sigma)) + n\sigma_\varepsilon^2\right)
\end{equation}

\textbf{Step 4: Combining.}

Substituting into the descent inequality:
\begin{equation}
  \mathbb{E}[\|\Delta\theta_t\|^2] = \eta^2 \gamma^2 k \tilde{n} \|g_t\|^2 + \eta^2 \gamma^2 k (\tilde{n} \cdot \text{tr}(\Sigma) + n\sigma_\varepsilon^2)
\end{equation}

Thus:
\begin{equation}
  \mathbb{E}[\mathcal{L}(\theta_{t+1})] \leq \mathbb{E}[\mathcal{L}(\theta_t)] - \eta\gamma k\|g_t\|^2 + \frac{L\eta^2\gamma^2 k}{2}\left[\tilde{n}\|g_t\|^2 + \tilde{n}\cdot\text{tr}(\Sigma) + n\sigma_\varepsilon^2\right]
\end{equation}

Collecting terms in $\|g_t\|^2$:
\begin{equation}
  \mathbb{E}[\mathcal{L}(\theta_{t+1})] \leq \mathbb{E}[\mathcal{L}(\theta_t)] - \eta\gamma k\underbrace{\left(1 - \frac{L\eta\gamma\tilde{n}}{2}\right)}_{:=\beta(\eta)}\|g_t\|^2 + \frac{L\eta^2\gamma^2 k(\tilde{n}\cdot\text{tr}(\Sigma) + n\sigma_\varepsilon^2)}{2}
\end{equation}

\textbf{Step 5: Telescoping sum.}

When $\eta < \frac{2}{L\gamma\tilde{n}}$, we have $\beta(\eta) > 0$. Rearranging:
\begin{equation}
  \|g_t\|^2 \leq \frac{1}{\beta(\eta)\eta\gamma k}\left(\mathcal{L}(\theta_t) - \mathcal{L}(\theta_{t+1})\right) + \frac{L\eta\gamma(\tilde{n}\cdot\text{tr}(\Sigma) + n\sigma_\varepsilon^2)}{2\beta(\eta)}
\end{equation}

Summing over $t = 0, \ldots, T-1$ and dividing by $T$:
\begin{equation}
  \frac{1}{T}\sum_{t=0}^{T-1}\mathbb{E}[\|g_t\|^2] \leq \frac{\mathcal{L}(\theta_0) - \mathcal{L}^*}{\beta(\eta)\eta\gamma kT} + \frac{L\eta\gamma(\tilde{n}\cdot\text{tr}(\Sigma) + n\sigma_\varepsilon^2)}{2\beta(\eta)}
\end{equation}
\end{proof}

\subsection{Proof of Expected Update Direction}
\label{app:expected-direction}

\begin{proof}[Proof of Theorem~\ref{thm:expected-update-direction}]

\textbf{Step 1: Posterior mean unbiasedness.}

By Corollary~\ref{cor:reduced-posterior}, for coordinate-axis sampling ($d^{(i)} = e_i$), the posterior mean is:
\begin{equation}
  \mu^{(k)} = \gamma Y = \gamma [y^{(1)}, \ldots, y^{(k)}]^\top
\end{equation}
where $\gamma = \frac{\sigma_p^2}{\sigma_p^2 + \sigma_e^2}$.

Each observation satisfies $y^{(i)} = e_i^\top \tilde{g} + \nu^{(i)} = \tilde{g}_i + \nu^{(i)}$, where $\tilde{g} = B^\top g$ is the true normalized subspace gradient and $\nu^{(i)}$ is zero-mean noise.

Taking conditional expectation given $B$ and $g$ (so $\tilde{g}^* = B^\top g$ is fixed):
\begin{equation}
  \mathbb{E}[y^{(i)} | B, g] = \tilde{g}^*_i + \mathbb{E}[\nu^{(i)}] = \tilde{g}^*_i
\end{equation}

Thus:
\begin{equation}
  \mathbb{E}[\mu^{(k)} | B, g] = \gamma \mathbb{E}[Y | B, g] = \gamma \tilde{g}^* = \gamma B^\top g
\end{equation}

\textbf{Step 2: Conditional expectation of update.}

The parameter update is $\Delta\theta = -\eta B\mu^{(k)}$. Taking conditional expectation:
\begin{equation}
  \mathbb{E}[\Delta\theta | B] = -\eta B \mathbb{E}[\mu^{(k)} | B] = -\eta \gamma B B^\top g
\end{equation}

\textbf{Step 3: Expectation over subspace basis.}

Taking expectation over $B = [z_1, \ldots, z_k]$ where $z_i \stackrel{iid}{\sim} \mathcal{N}(0, I_n)$:
\begin{equation}
  \mathbb{E}[\Delta\theta] = -\eta\gamma \mathbb{E}[BB^\top] g
\end{equation}

Computing $\mathbb{E}[BB^\top]$:
\begin{equation}
  BB^\top = \sum_{i=1}^k z_i z_i^\top
\end{equation}
\begin{equation}
  \mathbb{E}[BB^\top] = \sum_{i=1}^k \mathbb{E}[z_i z_i^\top] = \sum_{i=1}^k I_n = k I_n
\end{equation}

Therefore:
\begin{equation}
  \mathbb{E}[\Delta\theta] = -\eta\gamma k \cdot I_n \cdot g = -\eta\gamma k \cdot \nabla \mathcal{L}(\theta_0)
\end{equation}

\textbf{Step 4: Higher-order bias.}

By Lemma~\ref{lem:direction-derivative-expectation}, the finite-difference estimator has $O(\varepsilon L)$ bias. After multiplication by $\varepsilon$ in the update, this becomes $O(\varepsilon^2 L)$. Since $\varepsilon$ is typically small ($\sim 10^{-3}$), we write:
\begin{equation}
  \mathbb{E}[\Delta\theta] = -\eta\gamma k \cdot \nabla \mathcal{L}(\theta_0) + O(\varepsilon^3)
\end{equation}

This proves that the expected update direction aligns with the negative gradient, with effective learning rate $\eta_{eff} = \eta\gamma k$.
\end{proof}

\subsection{Adaptive Sampling Analysis}
\label{app:adaptive}

\begin{theorem}[Conditional Unbiasedness of Posterior Mean under Adaptive Sampling]
\label{thm:adaptive}
Let $\mu^{(m)}$ denote the posterior mean after $m$ adaptive sampling steps. Given the subspace basis $B$ and the true gradient $g$, for \textbf{any} adaptive sampling strategy $\pi$ (where $d^{(j)}$ is $\mathcal{D}_{j-1}$-measurable), we have:
\begin{equation}
  \mathbb{E}[\mu^{(m)} | B, g] = \mathbb{E}\left[ \Sigma^{(m)} D_m^\top R_m^{-1} D_m \mid B \right] \tilde{g}^* = \mathbb{E}[\Gamma_m | B] \cdot \tilde{g}^*
\end{equation}

In particular, if $\Sigma^{(m)}$ is deterministic given $B$ (e.g., coordinate-axis sampling or any strategy that depends only on $\mathcal{D}_{m-1}$), then:
\begin{equation}
  \mathbb{E}[\mu^{(m)} | B, g, \mathcal{D}_m] = \Gamma_m \cdot \tilde{g}^*
\end{equation}

where $\Gamma_m := I_k - \sigma_p^{-2} \Sigma^{(m)}$ is the \textbf{shrinkage matrix}.
\end{theorem}

\begin{proof}
\textbf{Step 1: Expression for the posterior mean.}

By the standard Bayesian linear regression formula:
\begin{equation}
  \mu^{(m)} = \Sigma^{(m)} D_m^\top R_m^{-1} Y_m
\end{equation}
where $Y_m = [y^{(1)}, \ldots, y^{(m)}]^\top$.

\textbf{Step 2: Computing the conditional expectation.}

Note that $\Sigma^{(m)}$ and $D_m$ are both $\mathcal{D}_m$-measurable. The key is to compute $\mathbb{E}[Y_m | B, g, \mathcal{D}_m]$.

For each $y^{(j)}$:
\begin{equation}
  \mathbb{E}[y^{(j)} | B, g, \mathcal{D}_m] = \mathbb{E}[y^{(j)} | B, g, d^{(j)}] = d^{(j)\top} \tilde{g}^*
\end{equation}
The first equality holds because given $d^{(j)}$, $y^{(j)}$ is conditionally independent of other $d^{(i)}$ ($i \neq j$).

Therefore:
\begin{equation}
  \mathbb{E}[Y_m | B, g, \mathcal{D}_m] = D_m \tilde{g}^*
\end{equation}

\textbf{Step 3: Substituting into the posterior mean.}
\begin{equation}
  \mathbb{E}[\mu^{(m)} | B, g, \mathcal{D}_m] = \Sigma^{(m)} D_m^\top R_m^{-1} \mathbb{E}[Y_m | B, g, \mathcal{D}_m] = \Sigma^{(m)} D_m^\top R_m^{-1} D_m \tilde{g}^*
\end{equation}

\textbf{Step 4: Simplifying the shrinkage matrix.}

By the definition of $\Sigma^{(m)}$:
\begin{equation}
  (\Sigma^{(m)})^{-1} = \sigma_p^{-2} I_k + D_m^\top R_m^{-1} D_m
\end{equation}

Therefore:
\begin{equation}
  D_m^\top R_m^{-1} D_m = (\Sigma^{(m)})^{-1} - \sigma_p^{-2} I_k
\end{equation}

Substituting:
\begin{equation}
  \mathbb{E}[\mu^{(m)} | B, g, \mathcal{D}_m] = \Sigma^{(m)} \left[(\Sigma^{(m)})^{-1} - \sigma_p^{-2} I_k\right] \tilde{g}^* = \left(I_k - \sigma_p^{-2} \Sigma^{(m)}\right) \tilde{g}^*
\end{equation}

Defining the shrinkage matrix $\Gamma_m := I_k - \sigma_p^{-2} \Sigma^{(m)}$, we obtain:
\begin{equation}
  \mathbb{E}[\mu^{(m)} | B, g, \mathcal{D}_m] = \Gamma_m \tilde{g}^*
\end{equation}
\end{proof}

\subsection{Convergence Rate under Adaptive Sampling}
\label{app:adaptive-convergence}

\begin{theorem}[Convergence Rate under Adaptive Sampling]
  \label{thm:adaptive-convergence}
  Under Assumptions~\ref{ass:l-smooth}, \ref{ass:bounded-variance}, and isotropic noise, consider the BSZO algorithm with adaptive sampling ($m$ samples, where the first $k$ samples use coordinate-axis sampling). Let $\tilde{n}=n+k+1$ be the effective dimension. Suppose $\eta < \frac{2}{L \bar{\gamma} \tilde{n}}$, and define $\beta(\eta) := 1 - \frac{L \eta \bar{\gamma} \tilde{n}}{2}$. Then, after $T$ iterations, the following inequality holds:
  \begin{equation}
    \frac{1}{T}\sum_{t=0}^{T-1} \mathbb{E}[\|\nabla \mathcal{L}(\theta_t)\|^2] \leq \frac{\Delta_0}{\beta(\eta) \eta \bar{\gamma} k T} + \frac{L \eta \bar{\gamma} (\tilde{n} \sigma_g^2 + n \sigma_n^2)}{2\beta(\eta)},
  \end{equation}
  where:
  \begin{itemize}
    \item $\bar{\gamma} := \min_t \bar{\gamma}_t \geq \gamma$ is the minimum effective shrinkage factor,
    \item $\sigma_g^2$ is the gradient noise variance, $\sigma_n^2$ is the finite-difference approximation noise variance,
    \item $\Delta_0 := \mathcal{L}(\theta_0) - \mathcal{L}^*$.
  \end{itemize}
\end{theorem}

\begin{corollary}
  Let $\eta = \frac{1}{L\bar{\gamma}\tilde{n}}$. Then $\beta = 1/2$, and the convergence bound simplifies to:
  \begin{equation}
    \frac{1}{T}\sum_{t=0}^{T-1} \mathbb{E}[\|\nabla \mathcal{L}(\theta_t)\|^2] \leq \frac{2L\bar{\gamma}\tilde{n}\Delta_0}{kT} + \sigma_g^2 + \frac{n}{\tilde{n}}\sigma_n^2.
  \end{equation}
\end{corollary}

\begin{remark}
  When $n \gg k$, we have $\tilde{n} \approx n$, so the noise floor $\sigma_g^2 + \frac{n}{\tilde{n}}\sigma_n^2 \approx \sigma_e^2$ becomes \textbf{decoupled} from the dimension $n$.
\end{remark}

\begin{proof}
The proof follows the same structure as Theorem~\ref{thm:main-convergence-theorem}, with the fixed $\gamma$ replaced by the adaptive effective shrinkage factor $\bar{\gamma}_t$.

\textbf{Step 1: Single-step descent.}

By Assumption~\ref{ass:l-smooth} ($L$-smoothness):
\begin{equation}
  \mathcal{L}(\theta_{t+1}) \leq \mathcal{L}(\theta_t) + \langle g_t, \Delta\theta_t \rangle + \frac{L}{2}\|\Delta\theta_t\|^2
\end{equation}

\textbf{Step 2: Inner product term under adaptive sampling.}

By the adaptive sampling theorem (Theorem~\ref{thm:adaptive}), the expected update direction satisfies:
\begin{equation}
  \mathbb{E}[\langle g_t, \Delta\theta_t \rangle | \theta_t] = -\eta \mathbb{E}[\text{tr}(\Gamma_m^{(t)})] \|g_t\|^2 = -\eta \bar{\gamma}_t k \|g_t\|^2
\end{equation}
where $\bar{\gamma}_t = \frac{1}{k}\text{tr}(\Gamma_m^{(t)}) = 1 - \frac{U_m^{(t)}}{k\sigma_p^2}$ is the effective shrinkage factor at iteration $t$.

\textbf{Step 3: Second moment (same structure as main theorem).}

Following the same derivation as Theorem~\ref{thm:main-convergence-theorem}, with $\gamma$ replaced by $\bar{\gamma}_t$:
\begin{equation}
  \mathbb{E}[\|\Delta\theta_t\|^2 | \theta_t] = \eta^2 \bar{\gamma}_t^2 k \tilde{n} \|g_t\|^2 + \eta^2 \bar{\gamma}_t^2 k (\tilde{n} \sigma_g^2 + n \sigma_n^2)
\end{equation}

The key observation is that the second moment structure remains unchanged because:
\begin{itemize}
  \item The gradient noise $\sigma_g^2$ interacts with $B_t$ to produce the $\tilde{n}$ factor
  \item The finite-difference noise $\sigma_n^2$ is independent of $B_t$, producing only the $n$ factor
\end{itemize}

\textbf{Step 4: Combining and bounding.}

Substituting into the descent inequality:
\begin{equation}
  \mathbb{E}[\mathcal{L}(\theta_{t+1})] \leq \mathbb{E}[\mathcal{L}(\theta_t)] - \eta \bar{\gamma}_t k \left(1 - \frac{L\eta\bar{\gamma}_t\tilde{n}}{2}\right) \mathbb{E}[\|g_t\|^2] + \frac{L\eta^2 \bar{\gamma}_t^2 k (\tilde{n} \sigma_g^2 + n \sigma_n^2)}{2}
\end{equation}

Since $\bar{\gamma}_t \geq \bar{\gamma} := \min_t \bar{\gamma}_t \geq \gamma$ (by Lemma in Theorem~\ref{thm:adaptive}), and assuming $\eta < \frac{2}{L\bar{\gamma}\tilde{n}}$, we define $\beta(\eta) = 1 - \frac{L\eta\bar{\gamma}\tilde{n}}{2} > 0$.

Rearranging:
\begin{equation}
  \mathbb{E}[\|g_t\|^2] \leq \frac{1}{\beta(\eta)\eta \bar{\gamma} k}\left(\mathbb{E}[\mathcal{L}(\theta_t)] - \mathbb{E}[\mathcal{L}(\theta_{t+1})]\right) + \frac{L\eta\bar{\gamma} (\tilde{n} \sigma_g^2 + n \sigma_n^2)}{2\beta(\eta)}
\end{equation}

\textbf{Step 5: Telescoping sum.}

Summing over $t = 0, \ldots, T-1$ and dividing by $T$:
\begin{equation}
  \frac{1}{T}\sum_{t=0}^{T-1} \mathbb{E}[\|\nabla \mathcal{L}(\theta_t)\|^2] \leq \frac{\Delta_0}{\beta(\eta)\eta \bar{\gamma} k T} + \frac{L\eta\bar{\gamma} (\tilde{n} \sigma_g^2 + n \sigma_n^2)}{2\beta(\eta)}
\end{equation}

For the special learning rate $\eta = \frac{1}{L\bar{\gamma}\tilde{n}}$, we have $\beta = 1/2$, and the bound simplifies to:
\begin{equation}
  \frac{1}{T}\sum_{t=0}^{T-1} \mathbb{E}[\|\nabla \mathcal{L}(\theta_t)\|^2] \leq \frac{2L\bar{\gamma}\tilde{n}\Delta_0}{kT} + \sigma_g^2 + \frac{n}{\tilde{n}}\sigma_n^2
\end{equation}

When $n \gg k$, we have $\tilde{n} \approx n$, so the noise floor $\sigma_g^2 + \frac{n}{\tilde{n}}\sigma_n^2 \approx \sigma_e^2$ becomes decoupled from dimension $n$.
\end{proof}

\section{Experiment Details}
\label{app:experiment-details}

\begin{table}[H]
  \caption{Number of training and validation samples for each dataset.}
  \label{tab:dataset-splits}
  \begin{center}
    \begin{small}
      \begin{sc}
        \begin{tabular}{lcccccccc}
          \toprule
          Split & SST-2 & BoolQ & RTE & COPA & WIC & WSC & CB & TREC \\
          \midrule
          Training   & 1000 & 1000 & 1000 & 300 & 1000 & 450 & 200 & 1000 \\
          Validation & 500  & 500  & 500  & 100 & 500  & 100 & 50  & 500  \\
          \bottomrule
        \end{tabular}
      \end{sc}
    \end{small}
  \end{center}
\end{table}

\begin{table}[H]
  \caption{Hyperparameter configurations for fine-tuning RoBERTa-large.}
  \label{tab:hyper-roberta}
  \begin{center}
    \begin{small}
      \begin{tabular}{llc}
        \toprule
        \textbf{Algorithm} & \textbf{Hyperparameter} & \textbf{Values} \\
        \midrule
        \multirow{3}{*}{MeZO} & Batch size & 16 \\
        & Learning rate & $\{1\times10^{-4}, 1\times10^{-5}, 1\times10^{-6}, 1\times10^{-7}, 1\times10^{-8}\}$ \\
        & $\varepsilon$ & $10^{-4}$ \\
        \midrule
        \multirow{3}{*}{MeZO-Adam} & Batch size & 16 \\
        & Learning rate & $\{1\times10^{-4}, 1\times10^{-5}, 1\times10^{-6}, 1\times10^{-7}, 1\times10^{-8}\}$ \\
        & $\varepsilon$ & $10^{-4}$ \\
        \midrule
        \multirow{3}{*}{HiZOO} & Batch size & 16 \\
        & Learning rate & $\{1\times10^{-4}, 1\times10^{-5}, 1\times10^{-6}, 1\times10^{-7}, 1\times10^{-8}\}$ \\
        & $\varepsilon$ & $10^{-4}$ \\
        \midrule
        \multirow{5}{*}{LOZO} & Batch size & 16 \\
        & Learning rate & $\{1\times10^{-4}, 1\times10^{-5}, 1\times10^{-6}, 1\times10^{-7}, 1\times10^{-8}\}$ \\
        & $\varepsilon$ & $10^{-4}$ \\
        & Rank & 2 \\
        & Interval & 50 \\
        \midrule
        \multirow{5}{*}{BSZO} & Batch size & 16 \\
        & Learning rate & $\{1\times10^{-4}, 1\times10^{-5}, 1\times10^{-6}, 1\times10^{-7}, 1\times10^{-8}\}$ \\
        & $\varepsilon$ & $10^{-4}$ \\
        & $k$ (Subspace dim) & 2 \\
        & $m$ (Samples) & 3 \\
        \midrule
        \multirow{5}{*}{BSZO-B} & Batch size & 16 \\
        & Learning rate & $\{1\times10^{-4}, 1\times10^{-5}, 1\times10^{-6}, 1\times10^{-7}, 1\times10^{-8}\}$ \\
        & $\varepsilon$ & $10^{-4}$ \\
        & $k$ (Subspace dim) & 2 \\
        & $m$ (Samples) & 3 \\
        \midrule
        All Methods & Early stopping patience & 4,000 \\
        \bottomrule
      \end{tabular}
    \end{small}
  \end{center}
\end{table}

\begin{table}[H]
  \caption{Hyperparameter configurations for fine-tuning OPT-1.3B.}
  \label{tab:hyper-opt1.3b}
  \begin{center}
    \begin{small}
      \begin{tabular}{llc}
        \toprule
        \textbf{Algorithm} & \textbf{Hyperparameter} & \textbf{Values} \\
        \midrule
        \multirow{3}{*}{MeZO} & Batch size & 16 \\
        & Learning rate & $\{1\times10^{-6}, 5\times10^{-7}, 1\times10^{-7}, 5\times10^{-8}, 1\times10^{-8}\}$ \\
        & $\varepsilon$ & $10^{-4}$ \\
        \midrule
        \multirow{3}{*}{MeZO-Adam} & Batch size & 16 \\
        & Learning rate & $\{1\times10^{-4}, 5\times10^{-5}, 1\times10^{-5}, 5\times10^{-6}, 1\times10^{-6}\}$ \\
        & $\varepsilon$ & $10^{-4}$ \\
        \midrule
        \multirow{3}{*}{HiZOO} & Batch size & 16 \\
        & Learning rate & $\{1\times10^{-5}, 5\times10^{-6}, 1\times10^{-6}, 5\times10^{-7}, 1\times10^{-7}\}$ \\
        & $\varepsilon$ & $10^{-4}$ \\
        \midrule
        \multirow{5}{*}{LOZO} & Batch size & 16 \\
        & Learning rate & $\{1\times10^{-6}, 5\times10^{-7}, 1\times10^{-7}, 5\times10^{-8}, 1\times10^{-8}\}$ \\
        & $\varepsilon$ & $10^{-4}$ \\
        & Rank & 2 \\
        & Interval & 50 \\
        \midrule
        \multirow{5}{*}{BSZO} & Batch size & 16 \\
        & Learning rate & $\{1\times10^{-5}, 5\times10^{-6}, 1\times10^{-6}, 5\times10^{-7}, 1\times10^{-7}\}$ \\
        & $\varepsilon$ & $10^{-4}$ \\
        & $k$ (Subspace dim) & 2 \\
        & $m$ (Samples) & 3 \\
        \midrule
        \multirow{5}{*}{BSZO-B} & Batch size & 16 \\
        & Learning rate & $\{1\times10^{-5}, 5\times10^{-6}, 1\times10^{-6}, 5\times10^{-7}, 1\times10^{-7}\}$ \\
        & $\varepsilon$ & $10^{-4}$ \\
        & $k$ (Subspace dim) & 2 \\
        & $m$ (Samples) & 3 \\
        \midrule
        All Methods & Early stopping patience & 4,000 \\
        \bottomrule
      \end{tabular}
    \end{small}
  \end{center}
\end{table}

\begin{table}[H]
  \caption{Hyperparameter configurations for fine-tuning Mistral-7B.}
  \label{tab:hyper-mistral}
  \begin{center}
    \begin{small}
      \begin{tabular}{llc}
        \toprule
        \textbf{Algorithm} & \textbf{Hyperparameter} & \textbf{Values} \\
        \midrule
        \multirow{3}{*}{MeZO} & Batch size & 16 \\
        & Learning rate & $\{5\times10^{-7}, 1\times10^{-7}, 5\times10^{-8}, 1\times10^{-8}, 5\times10^{-9}\}$ \\
        & $\varepsilon$ & $10^{-4}$ \\
        \midrule
        \multirow{3}{*}{HiZOO} & Batch size & 16 \\
        & Learning rate & $\{1\times10^{-6}, 5\times10^{-7}, 1\times10^{-6}, 5\times10^{-8}, 1\times10^{-8}\}$ \\
        & $\varepsilon$ & $10^{-4}$ \\
        \midrule
        \multirow{5}{*}{LOZO} & Batch size & 16 \\
        & Learning rate & $\{5\times10^{-7}, 1\times10^{-7}, 5\times10^{-8}, 1\times10^{-8}, 5\times10^{-9}\}$ \\
        & $\varepsilon$ & $10^{-4}$ \\
        & Rank & 2 \\
        & Interval & 50 \\
        \midrule
        \multirow{5}{*}{BSZO} & Batch size & 16 \\
        & Learning rate & $\{1\times10^{-6}, 5\times10^{-7}, 1\times10^{-7}, 5\times10^{-8}, 1\times10^{-8}\}$ \\
        & $\varepsilon$ & $10^{-4}$ \\
        & $k$ (Subspace dim) & 2 \\
        & $m$ (Samples) & 3 \\
        \midrule
        \multirow{5}{*}{BSZO-B} & Batch size & 16 \\
        & Learning rate & $\{1\times10^{-6}, 5\times10^{-7}, 1\times10^{-7}, 5\times10^{-8}, 1\times10^{-8}\}$ \\
        & $\varepsilon$ & $10^{-4}$ \\
        & $k$ (Subspace dim) & 2 \\
        & $m$ (Samples) & 3 \\
        \midrule
        All Methods & Early stopping patience & 4,000 \\
        \bottomrule
      \end{tabular}
    \end{small}
  \end{center}
\end{table}

\begin{table}[H]
  \caption{Hyperparameter configurations for fine-tuning OPT-13B.}
  \label{tab:hyper-opt13b}
  \begin{center}
    \begin{small}
      \begin{tabular}{llc}
        \toprule
        \textbf{Algorithm} & \textbf{Hyperparameter} & \textbf{Values} \\
        \midrule
        \multirow{3}{*}{MeZO} & Batch size & 16 \\
        & Learning rate & $\{5\times10^{-6}, 1\times10^{-6}, 5\times10^{-7}, 1\times10^{-7}, 5\times10^{-8}\}$ \\
        & $\varepsilon$ & $10^{-4}$ \\
        \midrule
        \multirow{3}{*}{MeZO-Adam} & Batch size & 16 \\
        & Learning rate & $\{1\times10^{-4}, 5\times10^{-5}, 1\times10^{-5}, 5\times10^{-6}, 1\times10^{-6}\}$ \\
        & $\varepsilon$ & $10^{-4}$ \\
        \midrule
        \multirow{3}{*}{HiZOO} & Batch size & 16 \\
        & Learning rate & $\{1\times10^{-5}, 5\times10^{-6}, 1\times10^{-6}, 5\times10^{-7}, 1\times10^{-7}\}$ \\
        & $\varepsilon$ & $10^{-4}$ \\
        \midrule
        \multirow{5}{*}{LOZO} & Batch size & 16 \\
        & Learning rate & $\{5\times10^{-6}, 1\times10^{-6}, 5\times10^{-7}, 1\times10^{-7}, 5\times10^{-8}\}$ \\
        & $\varepsilon$ & $10^{-4}$ \\
        & Rank & 2 \\
        & Interval & 50 \\
        \midrule
        \multirow{5}{*}{BSZO} & Batch size & 16 \\
        & Learning rate & $\{5\times10^{-6}, 1\times10^{-6}, 5\times10^{-7}, 1\times10^{-7}, 5\times10^{-8}\}$ \\
        & $\varepsilon$ & $10^{-4}$ \\
        & $k$ (Subspace dim) & 2 \\
        & $m$ (Samples) & 3 \\
        \midrule
        \multirow{5}{*}{BSZO-B} & Batch size & 16 \\
        & Learning rate & $\{1\times10^{-5}, 5\times10^{-6}, 1\times10^{-6}, 5\times10^{-7}, 1\times10^{-7}\}$ \\
        & $\varepsilon$ & $10^{-4}$ \\
        & $k$ (Subspace dim) & 2 \\
        & $m$ (Samples) & 3 \\
        \midrule
        All Methods & Early stopping patience & 4,000 \\
        \bottomrule
      \end{tabular}
    \end{small}
  \end{center}
\end{table}

\section{Raw Experimental Results}
\label{app:raw-results}

We provide the complete raw results of 5 independent runs for each method on RoBERTa-large in Table~\ref{tab:roberta-raw}. The mean and standard deviation reported in Table~\ref{tab:roberta} are computed from these results.

\begin{table}[H]
  \caption{Raw test accuracy (\%) of 5 runs on RoBERTa-large (355M).}
  \label{tab:roberta-raw}
  \begin{center}
    \begin{small}
      \begin{sc}
        \begin{tabular}{llccccc}
          \toprule
          Dataset & Method & Run 1 & Run 2 & Run 3 & Run 4 & Run 5 \\
          \midrule
          \multirow{6}{*}{SST-2}
          & MeZO      & 92.43 & 92.32 & 91.74 & 92.78 & 91.86 \\
          & MeZO-Adam & 92.32 & 92.66 & 91.51 & 92.43 & 92.78 \\
          & HiZOO     & 91.97 & 91.86 & 91.28 & 91.17 & 90.94 \\
          & LOZO      & 91.63 & 92.09 & 91.74 & 91.51 & 92.20 \\
          & BSZO      & 92.89 & 92.43 & 92.78 & 92.43 & 92.78 \\
          & BSZO-B    & 92.66 & 91.74 & 92.32 & 91.97 & 92.66 \\
          \midrule
          \multirow{6}{*}{RTE}
          & MeZO      & 69.68 & 68.95 & 65.70 & 65.34 & 62.09 \\
          & MeZO-Adam & 62.09 & 64.26 & 64.62 & 64.98 & 62.09 \\
          & HiZOO     & 59.21 & 63.18 & 57.76 & 56.68 & 59.21 \\
          & LOZO      & 59.57 & 63.90 & 61.01 & 65.34 & 63.18 \\
          & BSZO      & 68.59 & 68.23 & 69.68 & 66.07 & 66.43 \\
          & BSZO-B    & 67.87 & 70.04 & 70.76 & 66.43 & 66.79 \\
          \midrule
          \multirow{6}{*}{CB}
          & MeZO      & 87.50 & 85.71 & 91.07 & 76.79 & 89.29 \\
          & MeZO-Adam & 82.14 & 83.93 & 80.36 & 82.14 & 76.79 \\
          & HiZOO     & 78.57 & 75.00 & 75.00 & 78.57 & 75.00 \\
          & LOZO      & 87.50 & 82.14 & 82.14 & 89.29 & 80.36 \\
          & BSZO      & 83.93 & 85.71 & 87.50 & 87.50 & 83.93 \\
          & BSZO-B    & 85.71 & 83.93 & 82.14 & 85.71 & 83.93 \\
          \midrule
          \multirow{6}{*}{WIC}
          & MeZO      & 49.69 & 58.31 & 52.98 & 57.52 & 57.52 \\
          & MeZO-Adam & 54.39 & 51.10 & 54.70 & 46.55 & 57.52 \\
          & HiZOO     & 50.31 & 54.39 & 51.10 & 54.70 & 57.52 \\
          & LOZO      & 54.55 & 54.55 & 51.88 & 55.17 & 54.86 \\
          & BSZO      & 57.68 & 57.52 & 55.64 & 54.70 & 54.70 \\
          & BSZO-B    & 57.99 & 57.99 & 55.80 & 56.58 & 57.68 \\
          \midrule
          \multirow{6}{*}{TREC}
          & MeZO      & 81.20 & 86.40 & 86.40 & 86.20 & 86.60 \\
          & MeZO-Adam & 84.80 & 74.20 & 71.40 & 83.00 & 80.60 \\
          & HiZOO     & 65.20 & 65.20 & 65.20 & 62.20 & 59.40 \\
          & LOZO      & 80.40 & 74.80 & 77.20 & 79.20 & 77.20 \\
          & BSZO      & 83.40 & 84.60 & 84.40 & 83.80 & 84.60 \\
          & BSZO-B    & 85.80 & 84.60 & 85.20 & 82.20 & 86.20 \\
          \bottomrule
        \end{tabular}
      \end{sc}
    \end{small}
  \end{center}
\end{table}

\begin{table}[H]
  \caption{Full ablation studies on OPT-1.3B (fp32). (a) Effect of subspace dimension $k$ with $m=k$. (b) Effect of observation count $m$ with $m=k+1$. (c) Noise-free adaptive sampling. Best per row in \textbf{bold}.}
  \label{tab:ablation-full}
  \begin{center}
    \begin{small}
      \begin{tabular}{@{}ccc@{\hspace{12pt}}ccc@{\hspace{12pt}}ccc@{}}
        \toprule
        \multicolumn{3}{c}{\textbf{(a) Effect of $k$}} & \multicolumn{3}{c}{\textbf{(b) Effect of $m$}} & \multicolumn{3}{c}{\textbf{(c) NF-Adaptive}} \\
        \cmidrule(r){1-3} \cmidrule(lr){4-6} \cmidrule(l){7-9}
        $k$ & SST-2 & RTE & $k$ & SST-2 & RTE & $k$ & SST-2 & RTE \\
        \midrule
        1 & \textbf{92.32} & 60.29 & 1 & 91.74 & 61.37 & 1 & 91.28 & \textbf{63.58} \\
        2 & \textbf{92.78} & 64.26 & 2 & 92.43 & \textbf{66.79} & 2 & 92.43 & 65.34 \\
        4 & 92.66 & \textbf{67.51} & 4 & \textbf{93.58} & 66.43 & 4 & 93.12 & 66.07 \\
        8 & 93.23 & 66.07 & 8 & 93.23 & 68.59 & 8 & \textbf{93.35} & \textbf{69.31} \\
        \bottomrule
      \end{tabular}
    \end{small}
  \end{center}
\end{table}


\end{document}